\documentclass{article}

\PassOptionsToPackage{numbers, compress}{natbib}

    \usepackage[final]{neurips_2022}

\usepackage[utf8]{inputenc} 
\usepackage[T1]{fontenc}    
\usepackage{hyperref}       
\usepackage{url}            
\usepackage{booktabs}       
\usepackage{amsfonts}       
\usepackage{nicefrac}       
\usepackage{microtype}      
\usepackage{xcolor}         

\usepackage{arydshln}
\usepackage{qiangstyle}

\title{Sampling in Constrained Domains with Orthogonal-Space Variational Gradient Descent}

\author{%
  Ruqi Zhang\\
  Department of Computer Science\\
  Purdue University\\
  \texttt{ruqiz@purdue.edu} \\
   \And
   Qiang Liu \\
   Department of Computer Science \\
   University of Texas at Austin\\
   \texttt{lqiang@cs.texas.edu} \\
   \AND
   Xin T. Tong\\
  Department of Mathematics\\
  National University of Singapore\\
  \texttt{mattxin@nus.edu.sg} \\
}

\begin{document}

\maketitle

\begin{abstract}
Sampling methods, as important inference and learning techniques, are typically designed for unconstrained domains. However, constraints are ubiquitous in machine learning problems, such as those on safety, fairness, robustness, and many other properties that must be satisfied to apply sampling results in real-life applications. Enforcing these constraints often leads to implicitly-defined manifolds, making efficient sampling with constraints very challenging. In this paper, we propose a new variational framework with a designed orthogonal-space gradient flow (O-Gradient) for sampling on a manifold $\mathcal{G}_0$ defined by general equality constraints. O-Gradient decomposes the gradient into two parts: one decreases the distance to $\mathcal{G}_0$ and the other decreases the KL divergence in the orthogonal space. While most existing manifold sampling methods require initialization on $\mathcal{G}_0$, O-Gradient does not require such prior knowledge. We prove that O-Gradient converges to the target constrained distribution with rate $\widetilde{O}(1/\text{the number of iterations})$ under mild conditions. Our proof relies on a new Stein characterization of conditional measure which could be of independent interest. We implement O-Gradient through both Langevin dynamics and Stein variational gradient descent and demonstrate its effectiveness in various experiments, including Bayesian deep neural networks. 
\end{abstract}

\section{Introduction}
Sampling methods, such as Markov chain Monte Carlo (MCMC)~\citep{brooks2011handbook} and  Stein variational gradient descent (SVGD)~\citep{liu2016stein,liu2017stein}, have been widely used for getting samples from or approximating intractable distributions in machine learning (ML) problems, such as estimating Bayesian neural network posteriors~\citep{zhang2019cyclical}, generating new images~\citep{song2019generative}, and training energy-based models~\citep{lecun2006tutorial}. While being powerful, most sampling methods usually can only be used in unconstrained domains or some special geometric spaces. This greatly limits the application of sampling to many real-life tasks.

We consider sampling from a distribution $\pi$ with an equality constraint $g(x)=0$ where $g: \mathbb{R}^d\rightarrow \mathbb{R}$ is a general differentiable function. The domain in this case is the level set $\mathcal{G}_0=\{x\in \mathbb{R}^d: g(x)=0\}$ which is a submanifold in $\RR^d$.
We do not require additional information about $\mathcal{G}_0$, such as explicit parameterization or known in-domain points, which is in contrast, often demanded by previous methods~\citep{byrne2013geodesic,wang2020fast,brubaker2012family,lelievre2019hybrid}. The 
problem defined above includes many ML applications, such as disease diagnosis with logic rules constraint, policymaking with fairness constraint for different demographic subgroups, and autonomous driving with robustness constraint to unseen scenarios. 

In this paper, we propose a new variational framework which transforms the above constrained sampling problem into a constrained functional minimization problem. A special gradient flow, denoted  \emph{orthogonal-space gradient flow} (O-Gradient), is developed to minimize the objective. As illustrated in Figure~\ref{fig:illustration}a, the direction of O-Gradient $v$ can be decomposed into two parts: the first part $v_\sharp$ drives the sampler towards the manifold $\calG_0$ following $\nabla g$ and keeps it on $\calG_0$ once arrived; the second part $v\perpg$ makes the sampler explore $\calG_0$ following the density $\pi(x)$. We prove the convergence of O-Gradient in the continuous-time mean-field limit. O-Gradient can be applied to both Langevin dynamics and SVGD, resulting in O-Langevin and O-SVGD respectively.
As shown in Figure~\ref{fig:illustration}b\&c, both methods can converge to the target distribution on the manifold. In particular, O-Langevin converges following a noisy trajectory while O-SVGD converges smoothly, similar to their standard unconstrained counterparts.
We empirically demonstrate the sampling performance of O-Langevin and O-SVGD across different constrained ML problems. We summarize our contributions as follows:
\begin{itemize}
    \item We reformulate the hard-constrained sampling problem into a functional optimization problem and derive a special gradient flow, O-Gradient, to obtain the solution.
    \item We prove that O-Gradient converges to the target constrained distribution with rate $\widetilde{O}(1/\text{the number of iterations})$ under mild conditions. Our proof technique includes a new Stein characterization of conditional measure which could be of independent interest.
    \item We implement O-Gradient through both Langevin dynamics and SVGD and demonstrate its effectiveness in various experiments, including a constrained synthetic distribution, income classification with fairness constraint, loan classification with logic rules and image classification with robust Bayesian deep neural networks. 
\end{itemize}

\begin{figure}[t!]
    \centering
    \begin{tabular}{cccc}		
    		\hspace{-8mm}
    	\includegraphics[width=4.5cm]{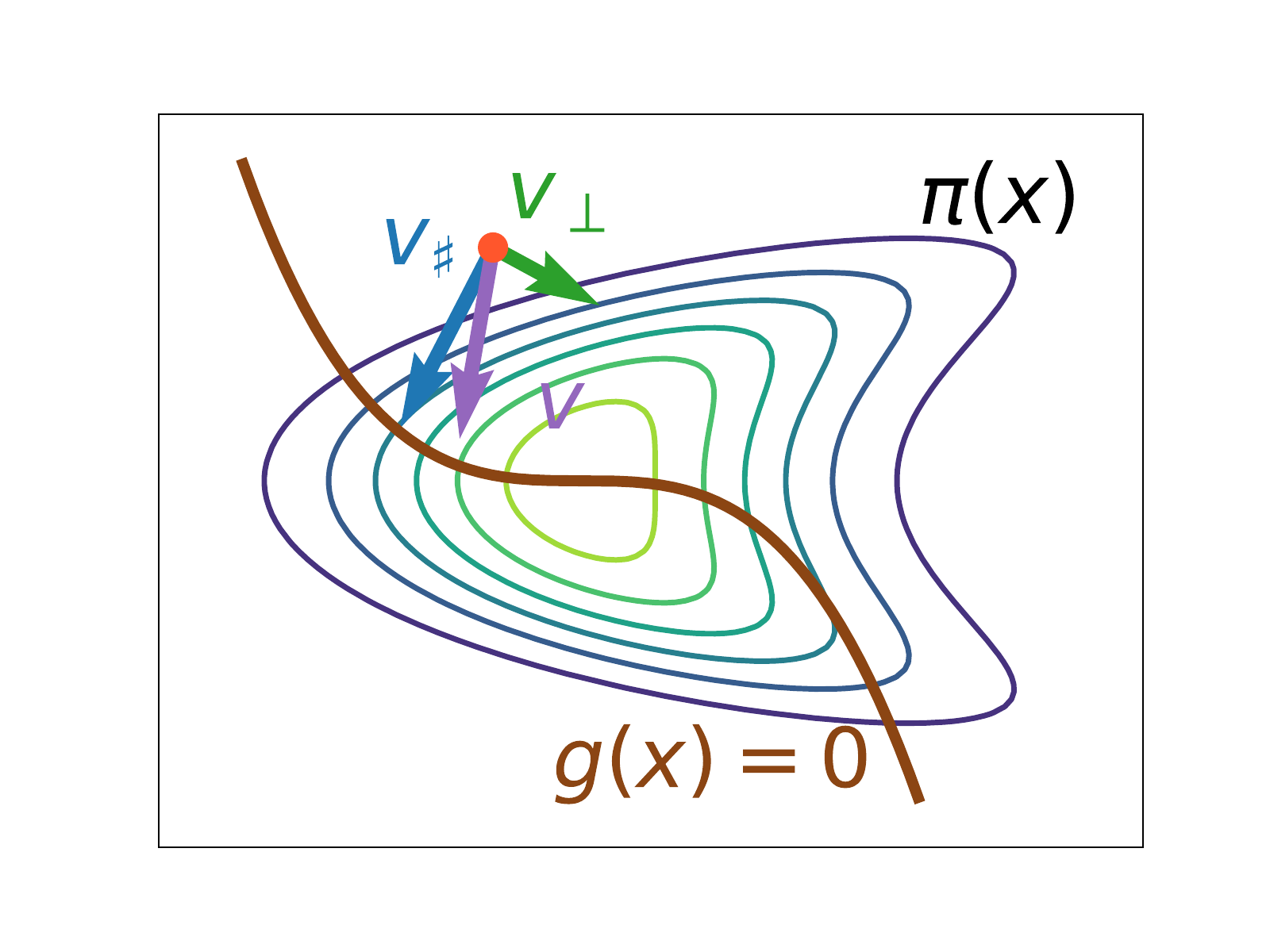}  &
    		\hspace{-8mm}
    	\includegraphics[width=4.5cm]{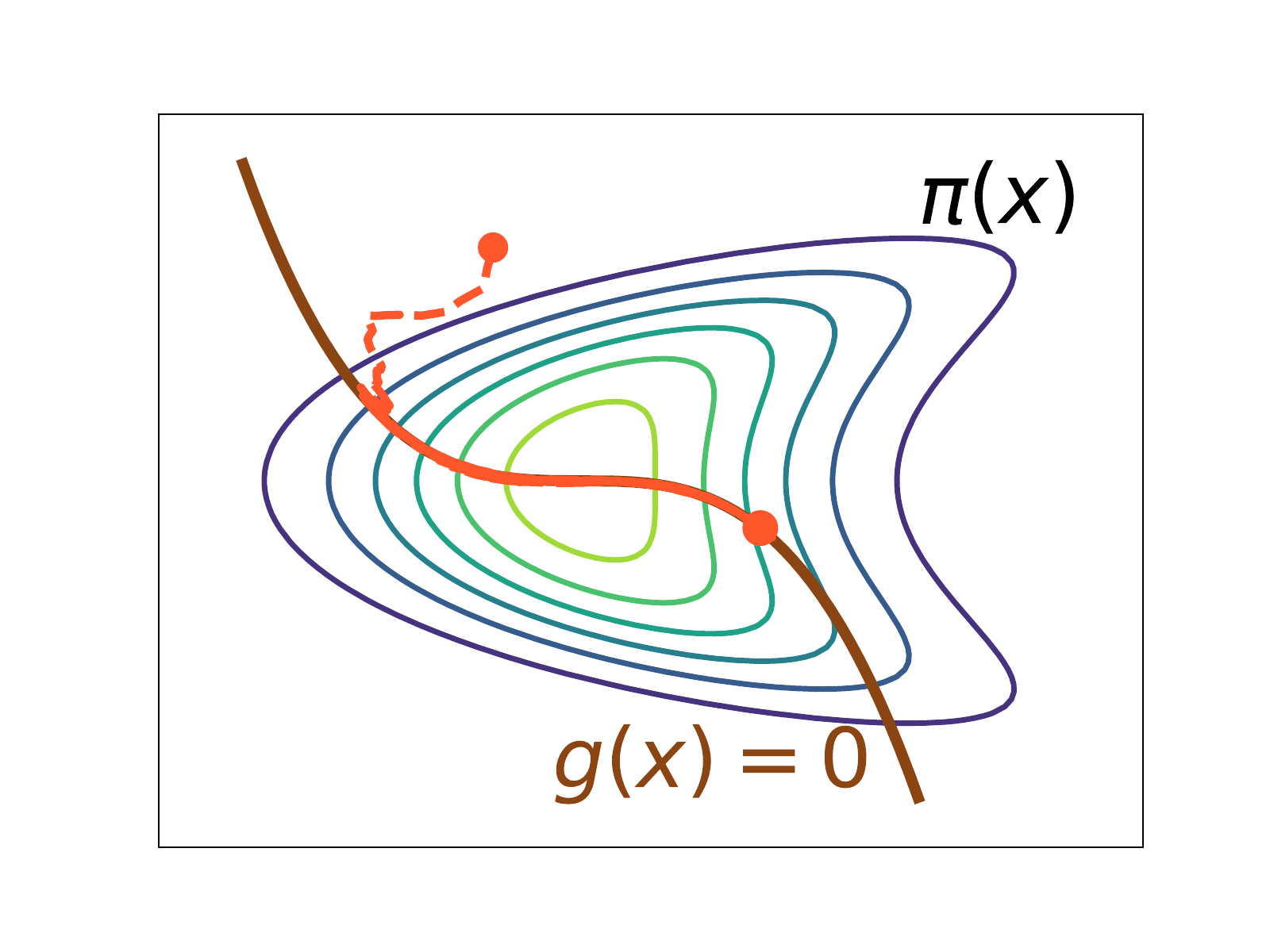} &
    	\hspace{-8mm}
    	\includegraphics[width=4.5cm]{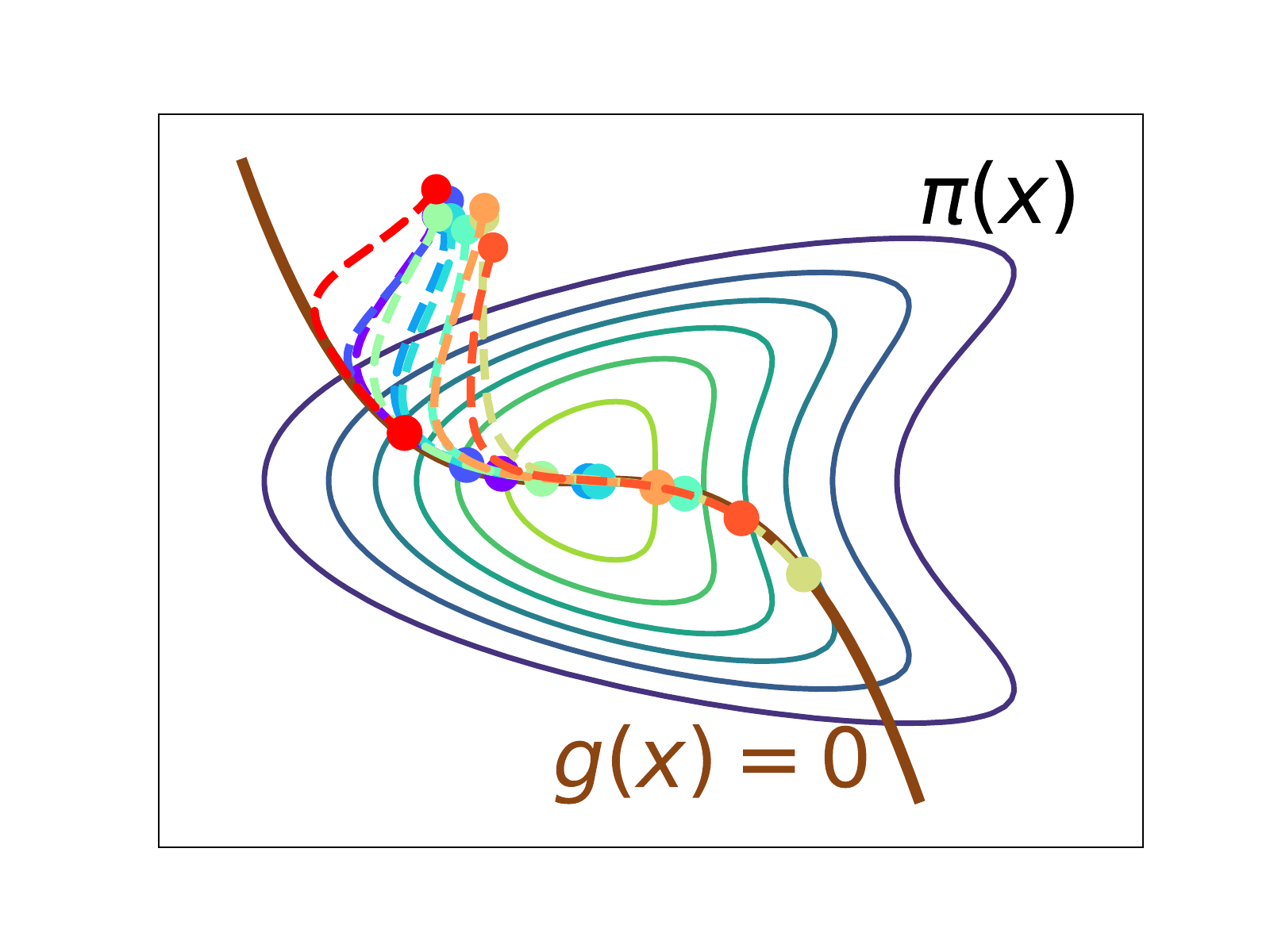}
    	\\	
    	(a) O-Gradient&
    	(b) O-Langevin&
    	(c) O-SVGD\\
    \end{tabular}
    \caption{Visualization of our methods. (a) O-Gradient $v$ is formed by $v_\sharp$ which follows $\nabla g$, and $v\perpg$ which is perpendicular to $\nabla g$. (b)-(c) Applying O-Gradient to Langevin dynamics and SVGD. Both methods can approach the manifold and sample on it.}
    \label{fig:illustration}
\end{figure}

\section{Related Work}

\paragraph{Sampling on Explicitly Defined Manifolds} 
Manifolds with special shapes, such as geometric or physics structures,  can sometimes be explicitly parameterized in lower dimension spaces. For example, a torus embedded in $\RR^3$ can be explicitly defined in two dimensions using polar coordinates. Variants of classical methods have been developed to sample on such manifolds, including rejection sampling~\citep{diaconis2013sampling}, Langevin dynamics~\citep{wang2020fast}, Hamiltonian Monte Carlo (HMC)~\citep{byrne2013geodesic} and Riemannian manifold HMC~\citep{lee2018convergence,patterson2013stochastic}. However, explicit parameterization is only applicable to a few special cases and cannot be used for general machine learning problems. In contrast to this line of work, our method is able to work with more general manifolds defined in the original domain $\RR^d$.

\paragraph{Sampling on Implicitly Defined Manifolds} Many common applications are not endowed with simple manifolds, such as molecular dynamics~\citep{lelievre2016partial}, matrix factorization~\citep{meyer2011linear} and free energy calculations~\citep{stoltz2010free}. Motivated by these applications, sampling methods on implicitly defined manifolds have been developed. \citet{brubaker2012family} has proposed a family of constrained MCMC methods by applying Lagrangian mechanics to Hamiltonian dynamics. \citet{zappa2018monte} has introduced a constrained Metropolis-Hastings (MH) with a reverse projection check to ensure the reversibility. Later, this method has been extended to HMC~\citep{lelievre2019hybrid} and multiple projections~\citep{lelievre2022multiple}. However, the implementation and analysis of these methods often assume the algorithm starts on the manifold and never leaves it, requiring prior known points on the manifold and expensive projection subroutines, such as Newton’s method~\citep{brubaker2012family,zappa2018monte,lelievre2012langevin,lelievre2019hybrid,lelievre2022multiple} or a long time ordinary differential equation (ODE)~\citep{zhang2020ergodic,sharma2021nonreversible,kook2022sampling}. In contrast, our method works with distributions supported on the ambient space and thus gets rid of the above strong assumptions, leading to a much faster update per iteration. This makes our method especially suitable for complex ML models such as deep neural networks.

\paragraph{Sampling with a Moment Constraint} 
Recently, sampling with a general moment constraint, such as $\E_q[g]\leq 0$ where $q$ is the approximated distribution, has been studied~\citep{liu2021sampling}. However, this type of constraint can not guarantee every sample to satisfy $g(x) = 0$. From a technical view, the target distribution with a moment constraint is usually not singular w.r.t. $\pi$, so the problem is conceptually less challenging compared to the problem considered in this work.

\section{Preliminaries}
\label{sec:prelim}
\paragraph{Variational Framework}
We review the derivation of Langevin dynamics and SVGD from a unified variational framework. The variational approach frames the problem of sampling into a KL divergence minimization problem: $\min_{q \in \mathcal P} \KL(q~||~\pi)$ where $\mathcal P$ is the space of probability measures. 
We start from an initial distribution $q_0$ and an initial point $x_0 \sim q_0$, and update $x_t$ following $\dno x_t = v_t(x_t) \dno t$, where $v_t\colon \RR^d\to \RR^d$ is a velocity field 
at time $t$. Then the density $q_t$ of $x_t$ follows Fokker-Planck equation: $\dno q_t/\dno t= - \nabla \cdot (v_t q_t)$, and the KL divergence decreases with the following rate \citep{liu2016stein}: 
\bbb \label{equ:kld}
-\frac{d}{dt} \KL(q_t ~||~ \pi) = 
\E_{q_t}[\stein_\pi v_t]=
\E_{q_t}[(s_\pi-s_{q_t})\tt v_t], \quad 
\eee 
where $\stein_\pi v(x)  = s_\pi(x) \tt v(x) + \dd\cdot v(x)$ is the Stein operator, and $s_p=\nabla \log p$ is the score function of the distribution $p$. The optimal $v_t$ is obtained by solving an optimization in a Hilbert space $\H$, 
\bbb\label{eq:pre-velocity}
\max_{v \in \H} \E_{q_t}[(s_\pi-s_{q_t})\tt v] - \frac{1}{2} \norm{v}_{\H}^2.
\eee
The above objective makes sure that $v_t$ decreases the KL divergence as fast as possible. 

\paragraph{Langevin Dynamics and SVGD Algorithms}
Both Langevin dynamics and SVGD can be derived from this variational framework by taking $\H$ to be different spaces.
Taking $\H$ to be $\calL_{q}^2$, the velocity field becomes
$v_t(\cdot)=\nabla s_\pi(\cdot) - \dd q_t(\cdot)$ which can be simulated by Langevin dynamics $dx_t=s_\pi(x_t)dt+dW_t$ with $W_t$ being a standard Browninan motion. After discretization 
with a step size $\eta>0$, the update step of Langevin dynamics is $x_{t+1} = x_t + \Langevin(x_t)$,  where
\bbb \label{eq:langevin-update}
\Langevin(x_t) = \eta\nabla\log\pi(x_t) + \sqrt{2\eta}\xi_t, ~~\xi_t\sim\mathcal{N}(0,I).
\eee 
Taking $\H$ to be the reproducing kernel
Hilbert space (RKHS) of a continuously differentiable kernel $k \colon \RR^d\times \RR^d\to \RR$, the velocity field becomes
$
v_t(\cdot) 
= \E_{x\sim q_t} [k_t(\cdot,x)s_\pi(x)  + \dd_x k_t(\cdot, x)].
$
After discretization, the update step of SVGD for particles $\{x_{i}\}_{i=1}^n$ is $x_{i,t+1} = x_{i,t} + \eta\cdot\SVGD_k(x_{i,t})$, for $i = 1,\ldots, n$, where $\eta$ is a step size and 
\begin{align}\label{eq:svgd-update}
    \SVGD_k (x_{i,t}) =\frac{1}{n}\sum_{j=1}^n k(x_{i,t},x_{j,t})\nabla_{x_{j,t}}\log\pi(x_{j,t})+\nabla_{x_{j,t}}k(x_{i,t},x_{j,t}).
\end{align}

\section{Main Method}

In this section, we formulate the constrained sampling problem into a constrained optimization through the variational lens in Section~\ref{sec:constrained-variational}, and introduce a new gradient flow to solve the problem in Section~\ref{sec:o-gradient}. We apply this general framework to Langevin dynamics and SVGD, leading to two practical algorithms in Section~\ref{sec:alg}.

\subsection{Constrained Variational Optimization}\label{sec:constrained-variational}
Recall that our goal is to draw samples according to the probability of $\pi$, but restricted to a low dimensional manifold specified by an equality: $\calG_0 \defeq \{x\in \RR^d \colon g(x) = 0\}$. Similar to the standard variational framework in Section~\ref{sec:prelim}, we can formulate the problem into a constrained optimization in the space of probability measures: 
\[\min_{q \in \mathcal P} \KL(q~||~\pi), ~~~\text{s.t.}~~~q(g(x) =0) = 1. 
\]
However, this problem is in general ill-posed. To see that, when $q$ satisfies the constraint, $q$ will be singular w.r.t. $\pi$, so both $\frac{dq}{d\pi}$ and $\KL(q~||~\pi)$ are not defined. Although the problem is ill-posed, we are actually still able to derive a $\KL$-gradient flow to solve the problem by considering $q$ supported on $\RR^d$. The intuition of the derivation is that, in addition to minimizing the objective as in Eq.~\eqref{equ:kld}, the velocity filed $v_t$ should also push $q$ towards $\calG_0$ to satisfy the constraint. Surprisingly, the distribution $q_t$ following such a gradient flow indeed converges to the target distribution on the manifold. We will focus on the derivation of the gradient flow here and leave its rigorous justification in Section~\ref{sec:theory}. 

\subsection{Orthogonal-Space Gradient Flow (O-Gradient)}\label{sec:o-gradient}
As mentioned above, besides maximizing the decay of $\KL(q~||~\pi)$, the velocity field $v_t$ also needs to drive  $q$ towards the manifold satisfying $g(x)=0$. 
In particular, we add to \eqref{eq:pre-velocity} a requirement that the value of $g(x)$ is driven towards $0$ with a given rate:
\begin{equation}
    \label{eq:vcontraint}
v_t=\arg\max_{v \in \H}\E_{q_t}[(s_\pi-s_{q_t})\tt v]  - \frac{1}{2} \norm{v}_{\H}^2,  ~~~\text{s.t.}~~~v_t(x)^\top\nabla g(x)=-\psi(g(x)) 
\end{equation}
where $\psi(x)$ is an increasing odd function. To see the effect of the constraint term, we consider three cases: 
\begin{itemize}
    \item When $g(x) > 0$, then $v_t(x)^\top\nabla g(x)=-\psi(g(x))>0$ which ensures that $v_t$ will make $g$ decrease strictly. 
    \item When $g(x) < 0$, then $v_t(x)^\top\nabla g(x)=-\psi(g(x))<0$ which ensures that $v_t$ will make $g$ increase strictly. 
    \item When $g(x) = 0$, then $v_t(x)^\top\nabla g(x)=-\psi(g(x))=0$ which ensures $x$ stay on the manifold $\mathcal{G}_0$.
\end{itemize} 
We choose $\psi(x)=\alpha \text{sign}(x)|x|^{1+\beta}$ with $\alpha>0$ and $\beta\in(0,1]$ in this paper because it is one of the simplest functions that satisfy the requirements and we found it works well in theory and practice. 

In summary, the objective function in Eq.~\eqref{eq:vcontraint} is the same as Eq.~\eqref{eq:pre-velocity} in the standard variational framework while the constraint ensures that $v_t$ pushes $q_t$ towards the manifold and keeps it stay.
It is easy to see that the solution of the above problem can be decomposed as $v_t=v_\sharp+v_\bot$ where
\bbb
\label{eq:vsharp} v_\sharp(x)=\frac{-\psi(g(x))\nabla g(x)}{\|\nabla g(x)\|^2}, ~~~v_\bot\bot \nabla g.
\eee
We use $f\bot g$ to denote (pointwise)  orthogonality: $f(x)\tt g(x) = 0$, $\forall x\in \RR^d$.  
Note that $v_\sharp$ is parallel to $\nabla g$ and the remaining is to determine $v_\bot$.  
Note that $v_\bot$ can be represented as a projection of an arbitrary function $u$ to the orthogonal space of $\dd g$: 
\bbb\label{eq:projD}
v_\bot = D(x)u(x),\quad \text{where } D(x):=I-\frac{\nabla g(x)\nabla g(x)^\top}{\|\nabla g(x)\|^2}.\eee 
The projection operator $D$ makes sure that $v_\bot\bot \nabla g$ holds for any $u$, of which the optimal value we can get by maximizing the unconstrained objective in Eq.~\eqref{eq:pre-velocity},
\bbb\label{eq:proju}
\max_{v\perpg} E_{q_t}[(s_\pi - s_{q_t}) \tt (v\parag + v\perpg)] - \tfrac{1}{2} \norm{v\perpg}_{\H}^2 \Rightarrow \max_u\E_{q_t}[(D(s_\pi - s_{q_t})) \tt u ] - \frac{1}{2} \norm{D u}_{\H}^2.
\eee
The optimal solution of  $u$ depends 
on the choice of space $\calH$, which we   discuss in Section~\ref{sec:alg}. 

Overall, we obtain the velocity field $v_t$ by first formulating a constrained optimization and then transforming it into an unconstrained optimization via orthogonal decomposition.
We call $v_t$ \emph{Orthogonal-Space Gradient Flow} (O-Gradient) and it drives $q_t$ to the target distribution 
only with the knowledge of $\dd g$,  requiring no explicit representation of the manifold $\calG_0$.

\subsection{Practical Algorithms}
\label{sec:alg}
After deriving O-Gradient for general Hilbert spaces $\calH$, we explain how to implement it using SVGD and Langevin dynamics. The resulting O-SVGD and O-Langevin are outlined in Algorithm~\ref{alg:o-sampling}. At a high level, our algorithms keep the original SVGD or Langevin dynamics movement in the directions perpendicular to $\nabla g$, while pushing the density towards $\calG_0$ along the $\nabla g$ direction. 

\paragraph{O-SVGD}
We apply O-Gradient to SVGD first since it is fairly straightforward. Recall that $v_\sharp$ can be obtained using Eq.~\eqref{eq:vsharp}. We solve Eq.~\eqref{eq:proju} to get $v_\bot$ through the following lemma. 

\begin{lem}
\label{lem:RKHSform}
When $\calH$ is an RKHS with kernel $k\colon \RR^d \times \RR^d \to \RR$, a solution to Eq.~\eqref{eq:proju} is  
$v\perpg = Du(x)=\E_{y\sim q_t}(k_\bot(x,y)s_\pi(y)+\nabla_y  \cdot k_\bot(x,y))$ with the orthogonal-space kernel $k_\bot(x,y)=k(x,y)D(x) D(y)$. 
Here $k_\bot\colon \RR^d \times \RR^d\to \RR^{d\times d}$ is matrix valued, and  
$\dd_y\cdot  k_\bot = \sum_{j}\partial_{y_j}k_\bot^{ij}(x,y)$. 
\end{lem}

Then the combined velocity is obtained using the original SVGD with the kernel $k_\bot$,
\begin{align*}
v_t(x)
&=v_\sharp(x)+\int k_\bot(x,y)s_\pi(y) q_t(y)dy+\int \nabla_y\cdot (k_\bot(x,y)) q_t(y)dy.
\end{align*}
Numerically, we iteratively update a set  of $n$ particles 
$\{x_{i, t} \}_{i=1}^n \subset \RR^d$, 
such that its empirical distribution  $\sum_{i=1}^n \delta_{\x_{i,t}}/n$ is an approximation of 
$q_t$
 in a proper sense when step size $\eta\to0$ and particle size $n\to+\infty$. Similar to the update of standard SVGD in Eq.~\eqref{eq:svgd-update}, the update of O-SVGD is $x_{i,t+1} = x_{i,t} + \eta\cdot \left (v_\sharp(x_{i,t}) +   \SVGD_{k\perpg}(x_{i,t})\right)$ where
\begin{align}\label{eq:o-svgd}
    \SVGD_{k\perpg} (x_{i,t}) &=\frac{1}{n}\sum_{j=1}^n k\perpg(x_{i,t},x_{j,t})\nabla_{x_{j,t}}\log\pi(x_{j,t})+\nabla_{x_{j,t}}k\perpg(x_{i,t},x_{j,t}).
\end{align}
It is worth noting that $\SVGD_{k\perpg}$ is identical to Eq.~\ref{eq:svgd-update} but with kernel $k_\bot$ rather than $k$.

\paragraph{O-Langevin}
The Langevin implementation requires some additional derivation. First of all, with $\calH=L^2_q$, we can show that the optimal  velocity field is 
$v_t(x)=\phi(x)-D(x)s_{q_t}(x)$ where 
$\phi(x)=v_\sharp(x)+D(x)s_\pi(x)$. This leads to a density flow
\begin{equation}
\label{eq:densityLD}
\frac{d}{dt} q_t(x)=-\nabla \cdot (\phi (x) q_t(x))+\nabla\cdot(D(x)\nabla q_t(x)).
\end{equation}
Next, we try to design a stochastic differential equation (SDE) of which the Fokker--Plank equation (FPE) is identical to Eq.~\eqref{eq:densityLD}. The result is given by the following:

\begin{thm}
\label{thm:LD}
Consider a vector field  $r(x)=\nabla\cdot D(x)$, or its component-wise formulation $r_i(x)=\sum_{j=1}^d \partial_{x_j}D_{i,j}(x)$, where $x_j$ denotes the $j$th dimension of $x$.
Consider the SDE 
\begin{equation}
\label{eq:CPGFLD}
dx_t=(\phi(x_t)+r(x_t))dt+\sqrt{2} D(x_t)dW_t \end{equation}
with $\phi(x)=v_\sharp(x)+D(x)s_\pi(x)$, then its FPE is identical to Eq.~\eqref{eq:densityLD}. 
Moreover, i)
the value $g(x_t)$ has deterministic decay $\frac{d}{dt}g(x_t)=-\psi(x_t)$; ii)
for any $f$ with $\nabla f \bot \nabla g=0$, the generator of $x_t$ satisfies the Langevin equation:
$\frac{d}{dt} \E[f(x_t)|x_0=x]\big |_{t=0} \defeq 
\mathcal{L} f(x)=\nabla f^\top(x) s_{\pi}(x)+\Delta f(x)$.
\end{thm}
It is worth pointing out that if $g(x_0)=0$, then $g(x_t)\equiv 0$, that is, $x_t$ always stays on $\calG_0$. In this case, Eq.~\eqref{eq:CPGFLD} degenerates to manifold Langevin dynamics studied in previous work~\citep{girolami2011riemann,wang2020fast}. However, our SDE does not have this requirement since \emph{it is still well-defined off $\calG_0$}. This is especially useful for numerical implementations, leading to a fast algorithm without expensive projection steps.

\begin{algorithm}[t]
  \caption{O-SVGD and O-Langevin.}
  \begin{algorithmic}
    \label{alg:o-sampling}
    \STATE \textbf{given:} Initialization $\{x_{i,0}\}_{i=1}^n$ for O-SVGD and $x_0$ for O-Langevin, step size $\eta$.
      \LOOP
      \STATE \textbf{compute} $v_\sharp(x)=\frac{-\psi(g(x))\nabla g(x)}{\|\nabla g(x)\|^2}$ 
      and the projection operator $D(x)=I-\frac{\nabla g(x)\nabla g(x)^\top}{\|\nabla g(x)\|^2}$.
    \STATE \textbf{if \texttt{O-SVGD}}, update $\{x_{i,t}\}_{i=1}^n$ by Eq.~\eqref{eq:o-svgd}.
    \STATE \textbf{if \texttt{O-Langevin}}, update $x_t$ by Eq.~\eqref{eq:o-langevin}.
    \ENDLOOP

  \end{algorithmic}
\end{algorithm}

Similar to the standard Langevin dynamics update in Eq.~\eqref{eq:langevin-update}, the update rule of O-Langevin is $x_{t+1} = x_t + \eta\cdot v_\sharp(x_t) + \Langevin\perpg(x_t)$ where
\begin{align}\label{eq:o-langevin}
    \Langevin\perpg(x_t) &= \eta D(x_t)s_\pi(x_t) + \eta r(x_t)+ \sqrt{2\eta}D(x_t)\xi_t,~~\xi_t\sim\mathcal{N}(0,I).
\end{align}

\section{Theoretical Analysis}\label{sec:theory}

We theoretically justify the convergence of O-Gradient in this section. To do so, we first describe the target measure as a conditioned measure $\Pi_0$, then derive its associated orthogonal-space Fisher divergence, and finally prove that O-Gradient converges to $\Pi_0$. 
\subsection{Conditioned measure and its Stein characterization}
First of all, conditioning on a zero-measure set is a challenging concept.
Assume we have a distribution $\Pi$ with density $\pi$. Let $A$ be a set with $\Pi(A)\neq 0$, then $\Pi(B|A) = \frac{\Pi(B\cap A)}{\Pi(A)}.$ However, if $\Pi(A) = 0$, this definition is ill-posed. 
Instead, a standard way to define a conditioned measure on a zero measure set is using disintegration theorem \citep{chang1997conditioning}. 
Let $\Pi_z(\,\cdot\,) = \Pi(\,\cdot\,| g(x) = z)$ be the measure such that 
\bbb
\label{eq:condmeas}
\E_\Pi[ f(x)] = 
\E_{z \sim \Pi^g} \E_{x\sim \Pi_z} [f(x)],~~\forall f,
\eee
where $\Pi^g$ is the pushforward measure of $\Pi$ under $g$. 
It is worth noting that $\Pi_0$ \emph{is not} the measure with Hausdorff density $\pi$ on $\calG_0$. Instead, under regularity conditions, $\Pi_0$ can also be defined through Hausdorff density $\pi(x)/|\nabla g(x)|$ on $\calG_0$, see Lemma 6.4.1 of \cite{simon2014introduction}.

One can also avoid this rather abstract definition, and consider 
an alternative definition which is more intuitive and Bayesian.  
Suppose we observe $z = g(x) + \sigma\xi$,  where $\xi\sim \normal(0,1)$. Under prior $x\sim\pi$, the posterior $\pi_{\sigma^2,z}(x) \propto \pi(x) \exp(-|z -g(x)|^2/2\sigma^2)$.   
Then we can show that $\Pi_z$ is the (weak) limit of $\pi_{\sigma^2,z}$ as $\sigma \to 0$, moreover, it satisfies  a \emph{Stein characterization of conditional measure}. We need  some regularity assumptions to state our results. 
   
\begin{mydef}\label{def:greg}
A density $q$ is $g$-regular if there is a constant $L$ such that for small enough $\eta>0$, 
\begin{equation}
\label{eq:greg}
|\E_{\xi\sim \mathcal{N}(0,1)}[q^g(z+\eta\xi)/q^g(z)]-1|\leq L\eta,\quad \forall z.
\end{equation}
\end{mydef}
Note that Eq.~$\eqref{eq:greg}$ holds if $q^g$ is a Lipschitz function. Moreover, while Definition~\ref{def:greg} requires Eq.~\eqref{eq:greg} to hold for all $z$, in practice we only concern $\Pi_0$, thus only need the equation to hold for $z$ near $0$. 

\begin{pro}
\label{pro:disint}
Suppose $\pi$ is $g$-regular. Then the weak   limit of $\pi_{\eta,z}(x)\propto \pi(x)\exp(-\frac1{2\eta} (g(x)-z)^2)$ 
as $\eta\to 0$ concentrates on $\calG_z$ and satisfies Eq.~\eqref{eq:condmeas}. Moreover, for each $z$,  
\bbb
\label{eq:projstein}
\E_{\Pi_z} \left [\stein_\pi \phi \right] =0,\quad \forall \phi\bot \nabla g. 
\eee
\end{pro}
The first part of Proposition \ref{pro:disint} also provides us with a way to sample $\Pi_z$, as we can sample a sequence of distribution $\pi_{\eta_k,z}$ with $\eta_k\to 0$. However, this method involves double loops and is usually much more expensive than single-loop algorithms.  

The second part of Proposition \ref{pro:disint} can be used to formulate an orthogonal-space Stein equation. Specifically, given a $q$ concentrated on $\calG_z$, we consider checking whether $\E_q [\stein_\pi \phi]$ is zero for $\phi\in \calH_\bot=\{\phi: \phi\bot \nabla g\}$, which is a necessary condition for $q=\Pi_z$. 
Notably, this equation is well defined in $\RR^d$ without using any parameterization of $\calG_z$.

Compared with the standard Stein equation,  we add an additional restriction that  $\phi\in\calH_\bot$.
We argue such restriction is actually very natural. To see that, we note $\Pi_z$ is concentrated on $\calG_z$, so we should check  the Stein equation with vector field that ``live'' only on $\calG_z$, namely the tangent bundle $\mathcal{T} \calG_z$. Notably, if $\phi(x)\in \mathcal{T} \calG_z$, then $\phi(x)\bot \dd g(x)$, therefore it is natural to consider $\calH_\bot$.

\subsection{Orthogonal-Space Fisher divergence}

Next we consider how to generalize Fisher divergence for constrained sampling. Since the orthogonal-space Stein equation can only discern discrepancy in $\calH_\bot$, it is natural to consider an orthogonal-space Fisher divergence 
\begin{equation}
\label{eq:projfisher}
F_{\bot}(q,\pi):=\|D(s_q-s_\pi)\|^2_{\calH},     
\end{equation}
where $D$ is the $\calH_\bot$-projection operator defined pointwisely using Eq.~\eqref{eq:projD}. Its easy to see $F_\bot$ is well defined for a density $q$ on $\RR^d$. But it remains unclear whether it can be used to measure the distance between $q$ and $\Pi_z$. 

To answer this question, we need some form of strong convexity conditions.
Recall that in the unconstrained case,  such conditions can be characterized by Poincar\'{e} inequalities (PI) when $\H$ is $L^2_q$, so we will focus on this setting. 
PI counterparts of SVGD remain an open question \cite{gorham2017measuring,duncan2019geometry}, 
which we hope to be addressed by future works.
 When constrained on $\calG_z$, PI's formulation involves gradients so it depends on the Riemannian metric used in  $\calG_z$. One natural choice of Riemannian metric is the metric inherited from $\RR^d$ which equals $\|D(x)\nabla f(x)\|$, since $D(x)$ is also the projection onto the tangent bundle of $\calG_z$. This leads to the following definition.

\begin{mydef}
$\Pi_z$ follows a $\kappa$-PI, if for any smooth $f$ on $\RR^d$, $\var_{\Pi_z} [f]\leq \kappa \E_{\Pi_z}[\|D\nabla f\|^2]$.
\end{mydef}
PI on Riemannian manifold has been studied in previous work~\citep{hebey2000nonlinear,li2006weighted}. In particular, Theorem 2.10 of \cite{hebey2000nonlinear} shows that PI holds for the Riemannian measure if the manifold $\calG_z$ is compact, and hence it also holds for any equivalent distributions.
But it is often hard to find the exact PI constant $\kappa$ for $\Pi_z$, especially if we are interested in a family of $\Pi_z$ with possibly discontinuous dependence on $z$.

\begin{pro}
\label{pro:meandiff}
Suppose that $\Pi_z$ satisfies $\kappa$-PI for $|z|\leq \delta$, and $q$ is supported on $\{x: |g(x)|\leq \delta\}$. Suppose also that $\Pi_z, q^g, \pi^g$ admit $C^1$ density functions.
Then for any function $f$ such that $|f|\leq 1$, the following holds
\begin{align*}
|\E_{q} [f]-\E_{\Pi_0} [f]|\leq &\sqrt{8\kappa\E_q [\|D(s_q-s_\pi)\|^2]} 
~+~ \max_{|z|\leq \delta} |\E_{\Pi_z} [f]-\E_{ \Pi_0} [f]|. 
\end{align*}
\end{pro}
In particular, Proposition \ref{pro:meandiff} shows that if we have a  $q$ such that it is 1) supported close to $\calG_0$ and 2) $F_\bot (q,\pi)$ is small, then $q$ is close to $\Pi_0$ in total variation. We can also interpret the bound alternatively as a decomposition of the difference between $q$ and $\Pi_0$: the first part measures the difference along $\calH_\bot$ directions and it is controlled by $F_\bot$; the second part measures the difference along $\nabla g$ direction, and it is controlled by the distance between $q$'s support  to $\calG_0$.

As a final remark, the orthogonal-space Stein equation \eqref{eq:projstein} and orthogonal-space Fisher divergence \eqref{eq:projfisher} can be applied to densities supported on $\RR^d$, and their formulations do not require information of $\calG_0$ such as parameterization and geodesic. Therefore, they can be used as very computationally friendly statistical divergence in practice.

\begin{pro}
\label{pro:flow}
Suppose we apply $v_t=v_\sharp+Du$ with \eqref{eq:vsharp} and \eqref{eq:proju} to the density field $q_t$.
\begin{enumerate}
    \item Let $M_t=\max\{|g(x)|,x\in \text{supp}(q_t)\} $. Let $S_t$ be the solution of an ordinary differential equation $\dot{S}_t=-\psi(S_t).$ Suppose $S_0=M_0<\infty$, then $M_t<S_t$ a.s..
    \item Assume $\psi$ is differentiable  with derivative  $\dot \psi$.  Suppose $g$ is also smooth enough so that
    \begin{equation}
        \label{eq:regular}
        \left|\nabla g^T s_\pi+\Delta g-\frac{2\nabla g^T \nabla^2 g \nabla g}{ \|\nabla g\|^2}\right|\leq C_0\|\nabla g\|^2.
    \end{equation}
\end{enumerate}
Then the KL-divergence follows $\frac{d}{dt}\KL(q_t\|\pi)\leq -F_\bot(q_t,\pi)+\E_{q_t} [|\dot{\psi}(g)|]+C_0\E_{q_t} [|\psi(g)|\|\nabla g\|^2].$
\end{pro}
The first result shows that the support of $q_t$ will shrink towards $\calG_0$. But in order to avoid possible stagnation case, e.g. $\|g\|$ is too small, it is necessary to impose \eqref{eq:regular}. 
The second result provides us some hint on the choice of $\psi$. Note that under $q_t$ for large enough $t$, $g(x)$ will take values close to $0$, so we need $\psi(0)=\dot{\psi}(0)=0$. Since $\psi(x)$ should have the same sign with $x$,  one natural choice would be $\psi(y)=\alpha\text{sign}(y)|y|^{1+\beta}$.
\begin{thm}
\label{thm:grad}
Suppose we choose $\psi(y)=\alpha\text{sign}(y)|y|^{1+\beta}$ for an $\alpha>0,\beta\in (0,1]$ and the conditions of Proposition \ref{pro:flow} hold. Suppose also that $\KL(q_0,\pi)<\infty$ and $\|\nabla g(x)\|\leq D_0$ if $|\psi(x)|<M_0$,  then 
$M_T=O(T^{-\frac1\beta}),$
$\min_{t\leq T} F_\bot(q_t,\pi)=O(\log T/T)$.
The orthogonal-space Stein equality holds approximately, in the sense that for any $\phi(x)\bot \nabla g(x)$ and $\|\phi\|_{\H}=1$, 
$\min_{t\leq T}|\E_{q_t} \stein_{\pi} \phi|\leq \sqrt{F_\bot(q_t,\pi)}=O(\sqrt{\log T/T}).$
\end{thm}
Theorem \ref{thm:grad} shows that we can obtain a $q_t$ that is supported on $\calG_{[-\delta,\delta]}=\{x:|g(x)|\leq \delta\}$, while it has close to zero orthogonal-space Fisher divergence and the  Stein identity holds approximately on $\calH_\bot$. Combining it with Proposition \ref{pro:meandiff}, we have 
\begin{cor}
Under the conditions of Theorem \ref{thm:grad}, for a bounded $f$ with $\calH=L^2_q$, suppose
 $h_f(z)=\E_\pi [f|g(X)=z]=\E_{\Pi_z} [f]$ is Lipschitz with near $z=0$, then 
$\min_{t\leq T}|\E_{q_t} [f]-\E_{\Pi_0} [f]|\leq O(\sqrt{\log T/T}).$
\end{cor}

\section{Experiments}\label{sec:exp}
We demonstrate the effectiveness of our methods on various tasks, including a constrained synthetic distribution, an income classification with a fairness constraint, a loan application with a logic constraint and Bayesian deep neural networks with a robustness constraint. For all MCMC methods, we run $n$ parallel chains and collect the final samples. For all SVGD, we use $n$ particles with an RBF kernel. We released the code at \url{https://github.com/ruqizhang/o-gradient}.

\paragraph{Synthetic Distribution}
We first demonstrate our methods on a two-dimensional synthetic distribution where the ground truth samples are available. Let $y\sim\mathcal{N}(0,I)$, and we transform $y$ to $x=[x_1,x_2]$ by $x=\phi^{-1}(y)$ where 
$\phi(x) =[
x_1 + x_2^3, 
x_2]\tt$. 
We aim to sample from $\pi(x)$ with constraint $g(x) = x_1 + x_2^3=0$. We are able to generate  ground truth samples by first generating $y_0\sim\mathcal{N}(0,1)$ and then setting $x_{\text{gt}}=[-y_0^3,y_0]$. We report the energy distance between the collected samples and the ground truth samples (measuring how well samples approximate the target distribution) and the mean absolute error (MAE) $|g(x)|$ (measuring how well the samples satisfy the constraint). We set $n=50$ for this task.

In order to compare with previous manifold sampling methods, which all require initializations on the manifold, we first consider the case when the sampler starts on the manifold. We compare with constrained Langevin dynamics (CLangevin) and constrained Hamiltonian Monte Carlo (CHMC), which are two advanced manifold samplers~\citep{lelievre2019hybrid}. From Figure~\ref{fig:toy}a, we observe that O-SVGD and O-Langevin converge after 2000 and 4000 steps respectively whereas CLangevin and CHMC have not fully converged even after 5000 steps. These results indicate that projection-based methods could be inefficient due to finding sub-optimal solutions in the projection subroutines. In terms of runtime, we observe that O-Langevin converges the fastest (see Appendix for comparisons). Figure~\ref{fig:toy}b shows that previous methods have almost zero MAE all the time since they are required to stay on the manifold. Our methods have reasonably small MAE (but not zero due to operating in the ambient space) even without expensive projection steps. O-Langevin has more wiggly trajectories than O-SVGD due to injecting the Gaussian noise.

Next, we test our methods starting outside the manifold, where previous manifold sampling methods are not applicable. From Figure~\ref{fig:toy}d\&e, we observe again that O-SVGD converges the fastest in terms of iterations and both methods converge very quickly to the manifold and are able to stay close to it.

\begin{figure}[t!]
    \centering
    \begin{tabular}{cccc}		
    		\hspace{-6mm}
    	\includegraphics[width=3.5cm]{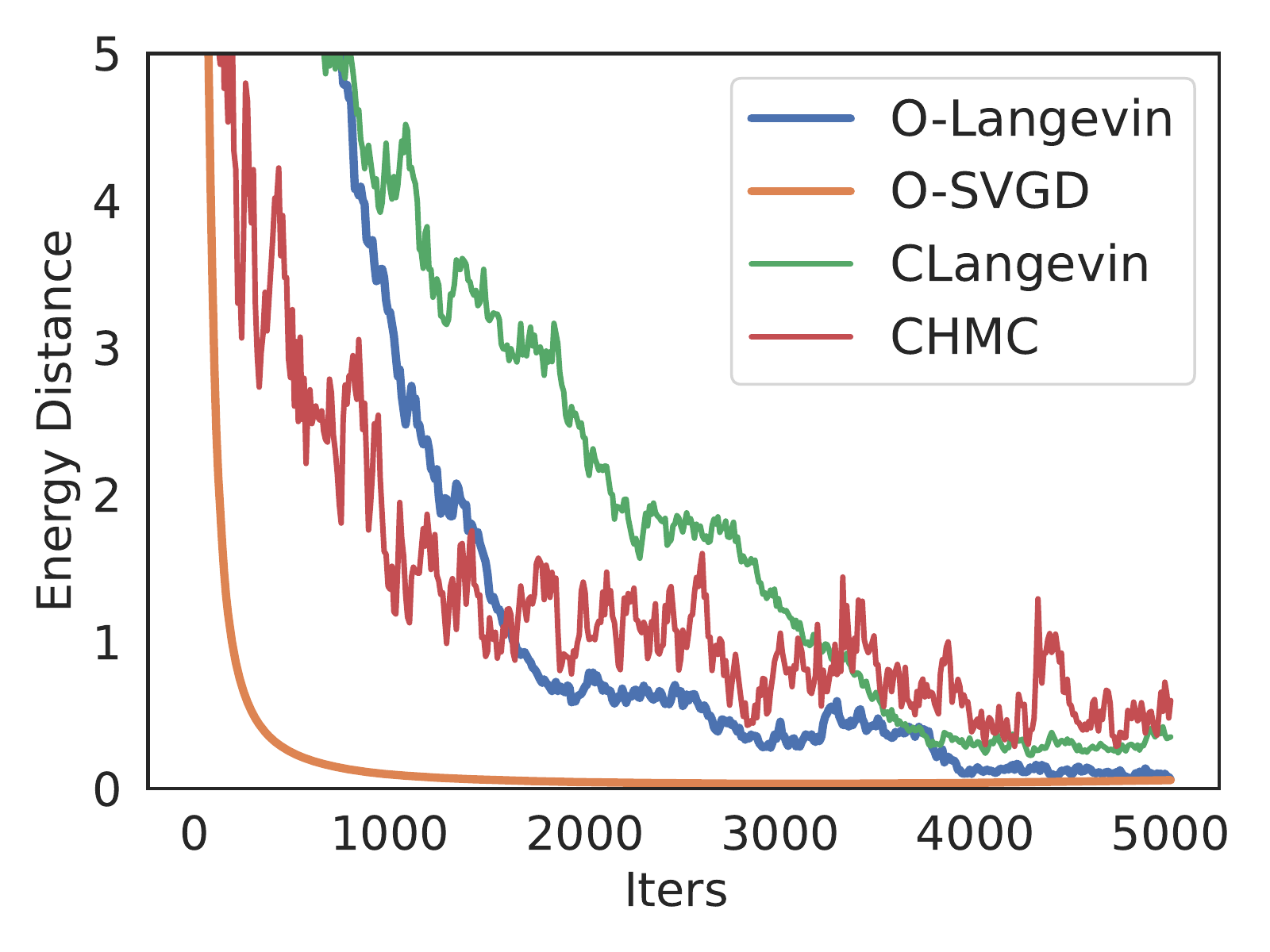}  &
    		\hspace{-6mm}
    	\includegraphics[width=3.5cm]{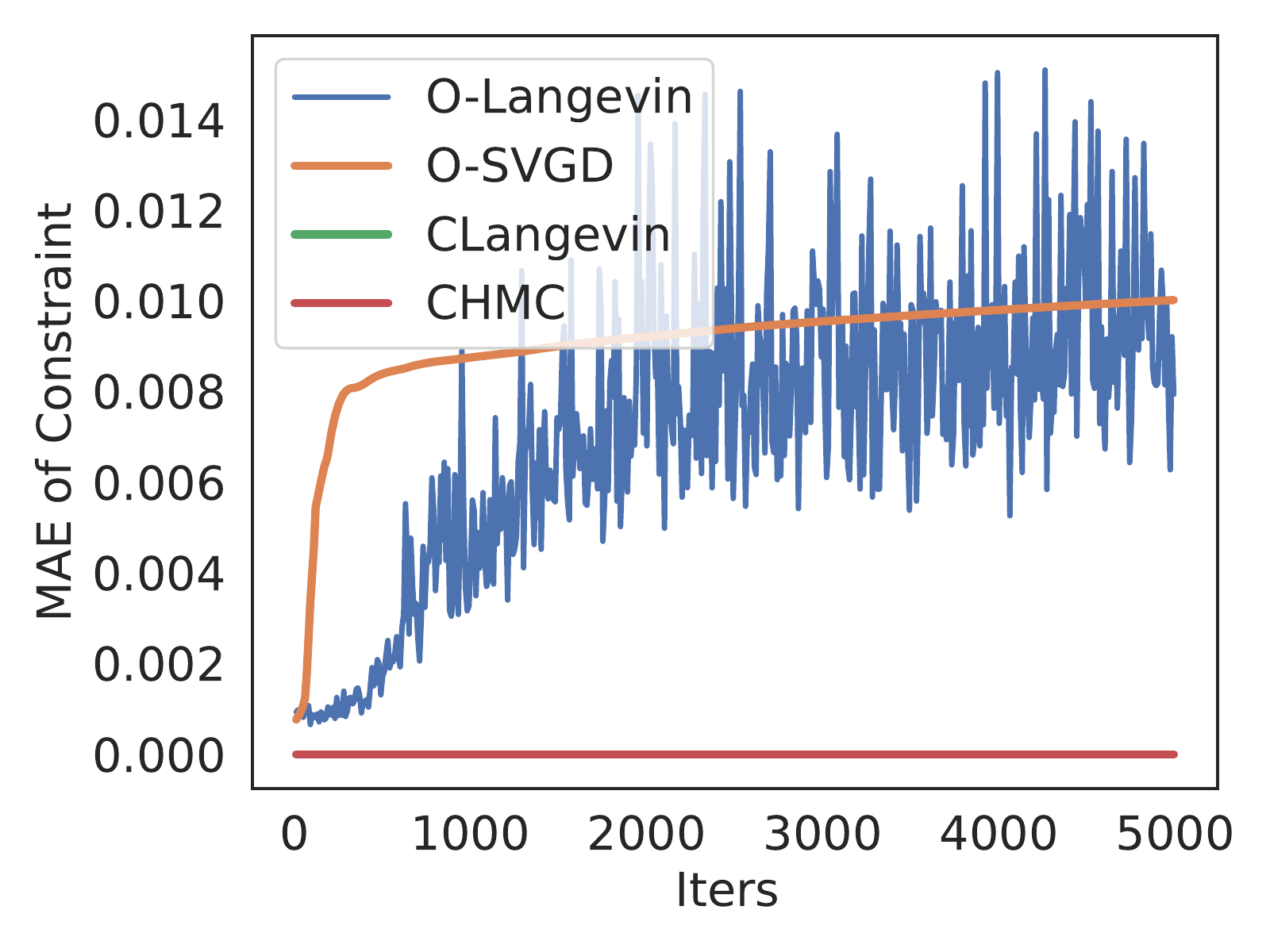} &
    	\hspace{-6mm}
    	\includegraphics[width=3.5cm]{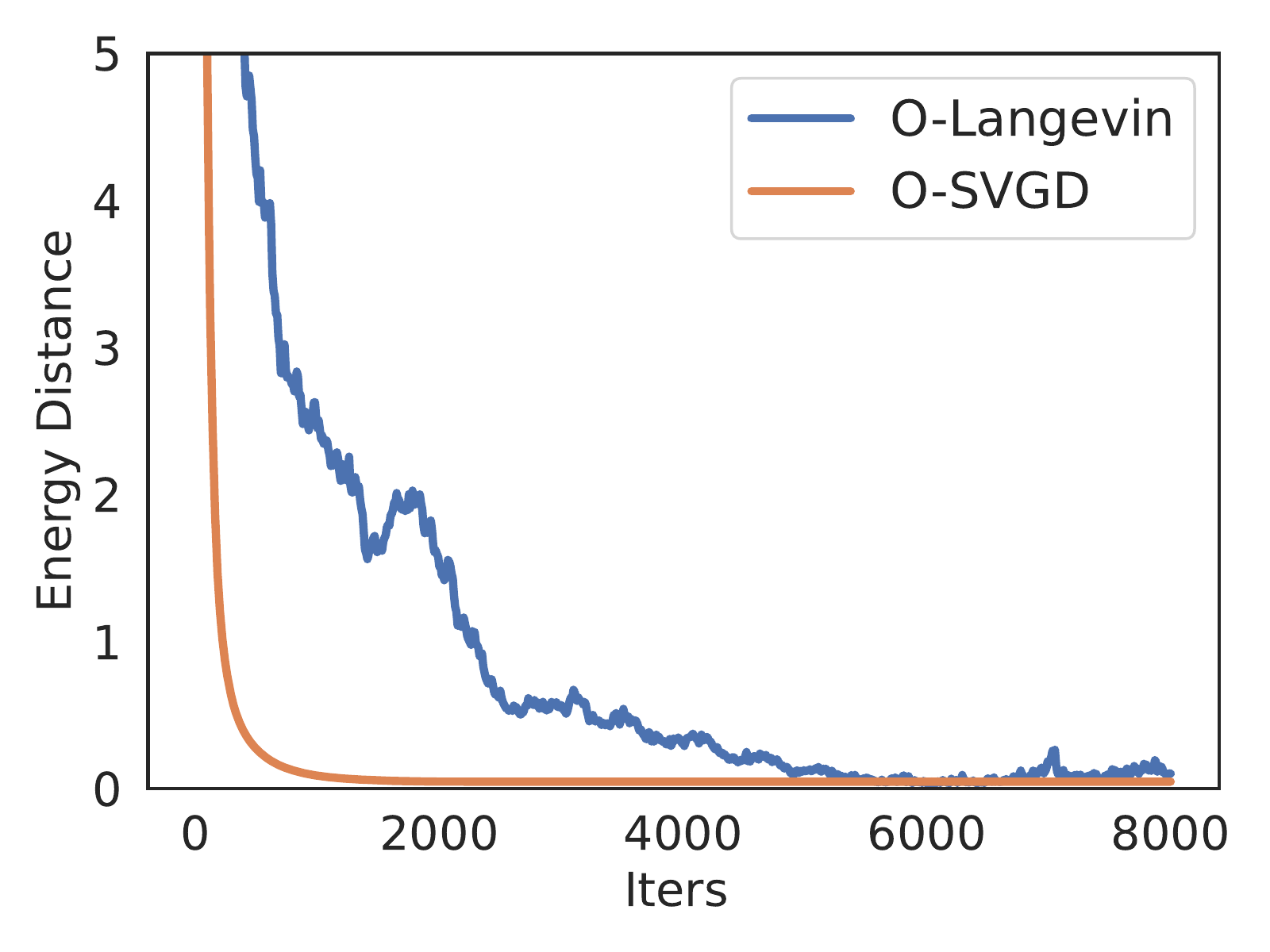}  &
    		\hspace{-6mm}
    	\includegraphics[width=3.5cm]{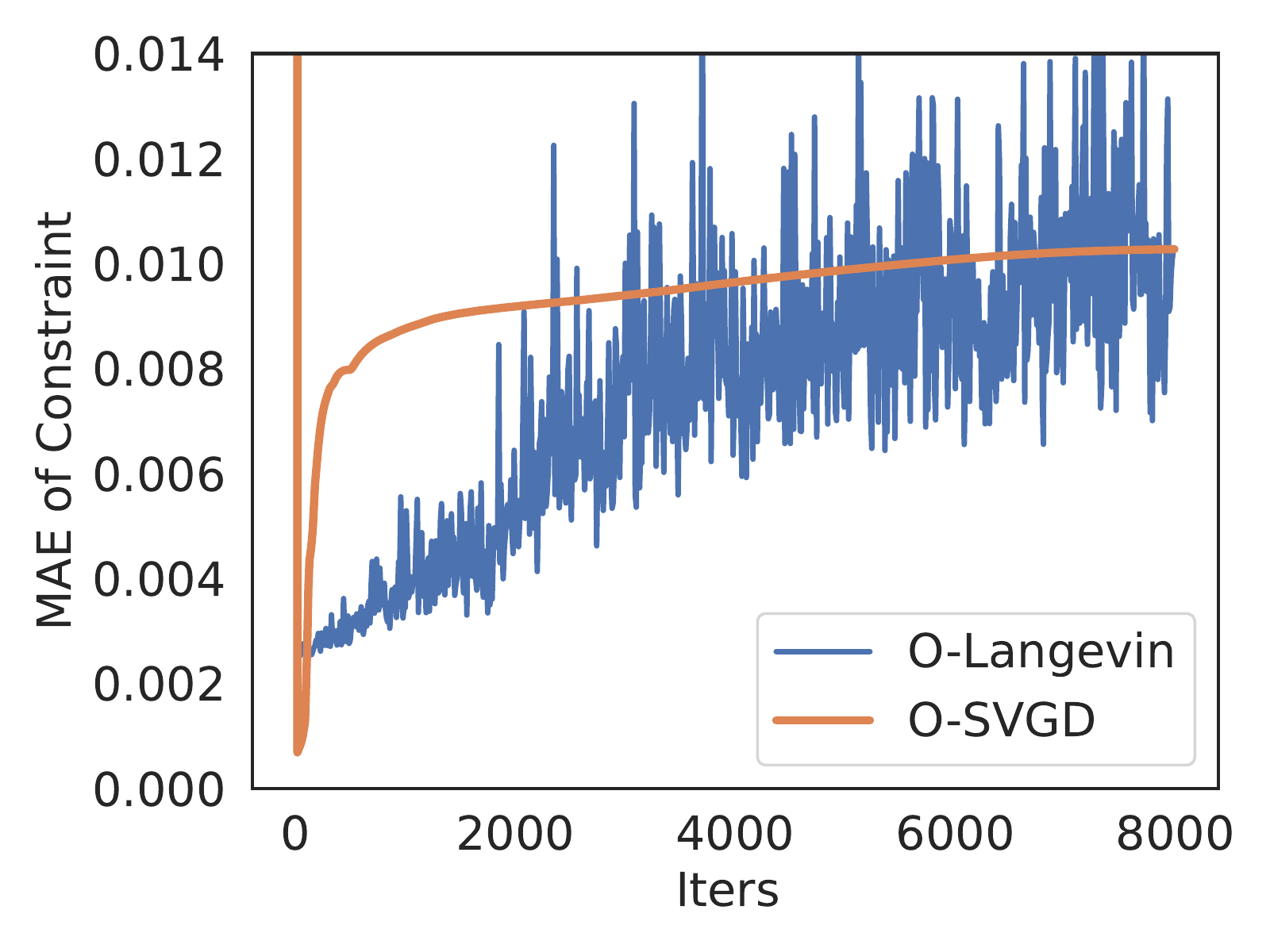}
    	\\	
    	(a)&
    	(b)&
    	(c)&
        (d)\\
    \end{tabular}
    \caption{Energy distance and MAE w.r.t. iterations when starting: (a)\&(b) on the manifold and (c)\&(d) outside the manifold.}
    \label{fig:toy}
\end{figure}

\paragraph{Income Classification with Fairness Constraint}
ML Fairness has gained increasing attention in recent years, which aims to guarantee unbiased treatments for individuals w.r.t. gender, race, disabilities, etc. Following previous work~\citep{martinez2020minimax,liu2020accuracy}, we predict whether an individual's annual income is greater than $50,000$ unfavorably in terms of the gender. We use \emph{Adult Income} dataset~\citep{kohavi1996scaling} and train a Bayesian neural network with two-layer multilayer perceptron (MLP). During testing, we use CV score, a standard fairness measure of disparate impact, as the evaluation metric~\citep{calders2010three}. On this task, CV score is defined as $|p(\text{income } 50,000|\text{female}) - p(\text{income } 50,000|\text{male})|$. Training and testing performance of our methods and standard algorithms are shown in Figure~\ref{fig:fairness}. We find that during training and testing, both O-Langevin and O-SVGD satisfy the constraint very well while maintaining high training LL and test accuracy. In contrast, standard Langevin and SVGD violate the constraint significantly, indicating that they fail to make fair predictions. These results demonstrate the power of O-Gradient which keeps the sampler on the manifold while still correctly sampling from $\pi(x)$.

\begin{figure}[t!]
    \centering
    \begin{tabular}{cccc}		
    	\hspace{-6mm}
    	\includegraphics[width=3.6cm]{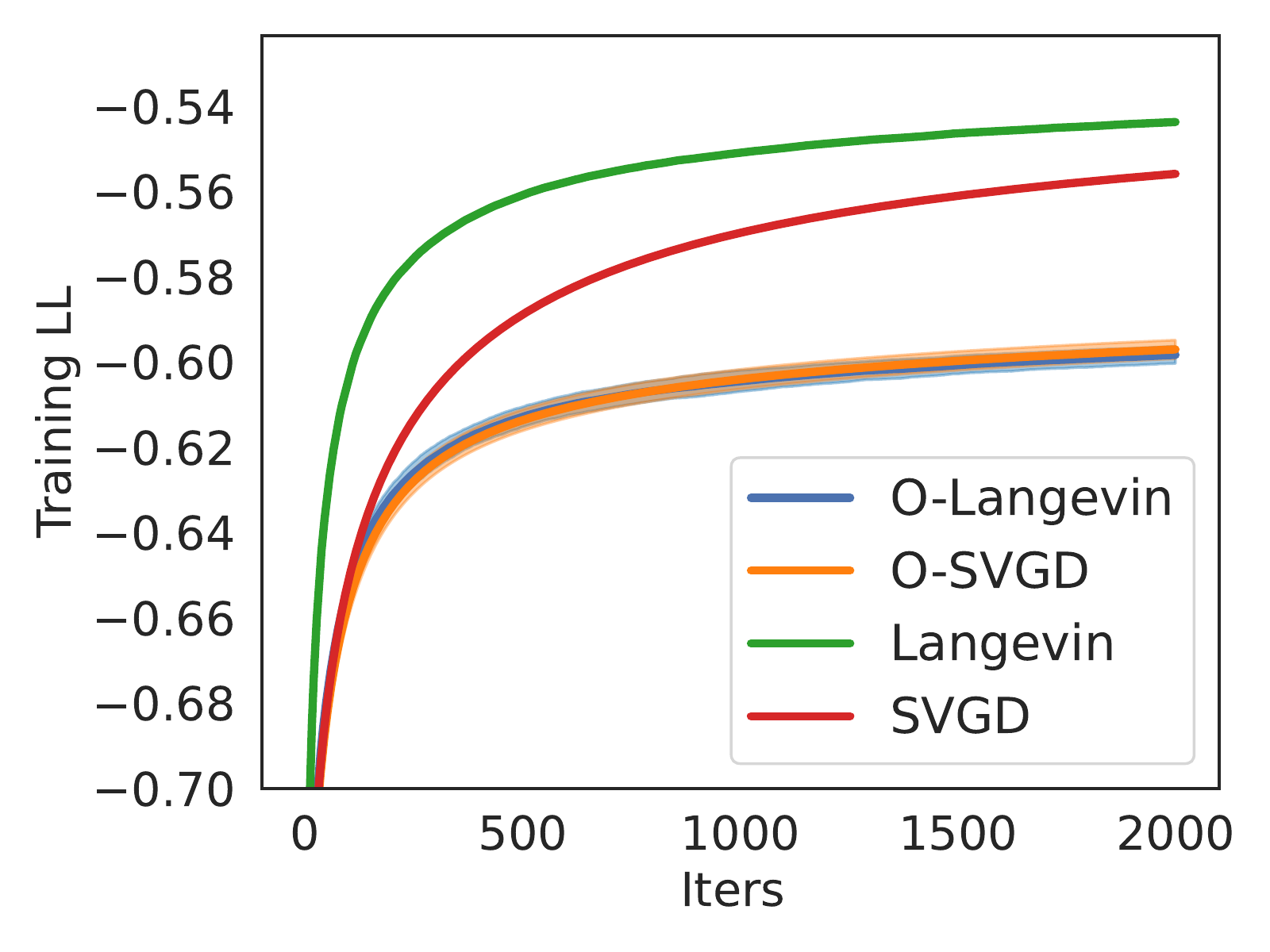}  &
    		\hspace{-6mm}
    	\includegraphics[width=3.6cm]{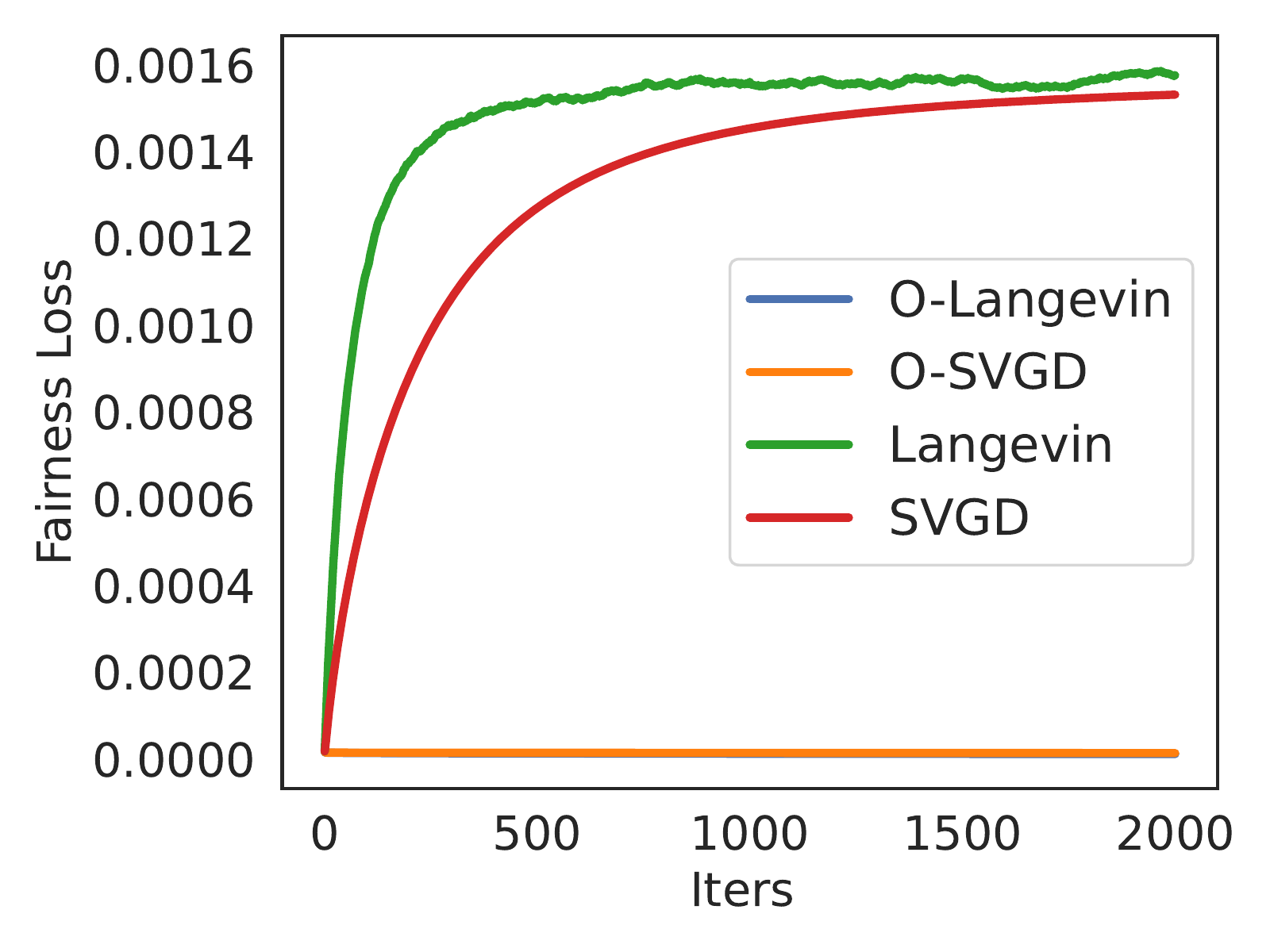} &
    	\hspace{-6mm}
    	\includegraphics[width=3.6cm]{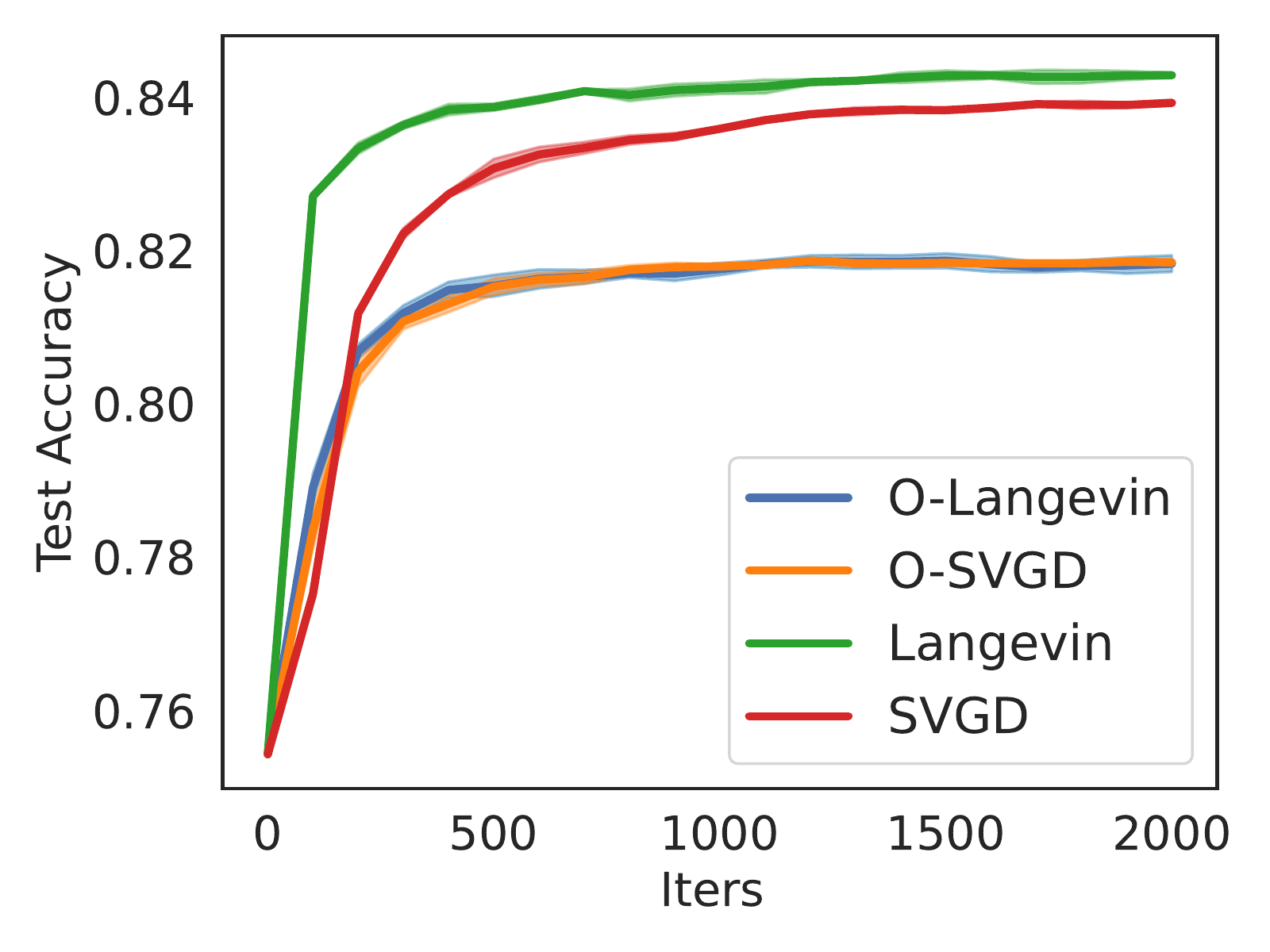} &
    	\hspace{-6mm}
    	\includegraphics[width=3.6cm]{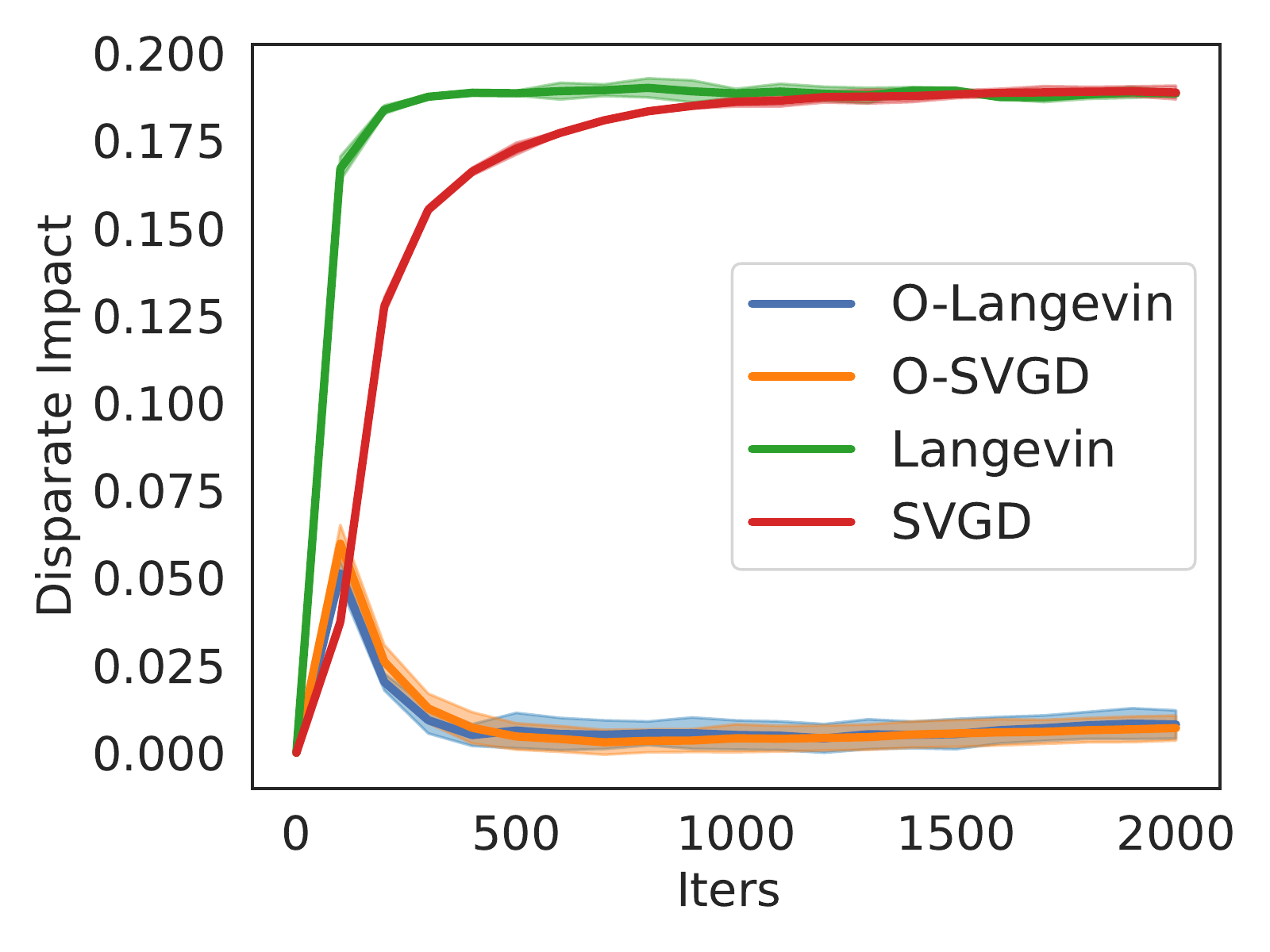}
    	\hspace{-0mm}\\	
    	\multicolumn{2}{c}{(a) Training Curve} &  \multicolumn{2}{c}{(b) Testing Curve}
    \end{tabular}
    \caption{Income Classification: O-Langevin and O-SVGD have almost zero fairness loss and disparate impact
while maintaining great prediction performance. Standard Langevin and SVGD fail to make fair predictions.}
    \label{fig:fairness}
\end{figure}
\begin{figure}[t!]
    \centering
    \begin{tabular}{cccc}		
    		\hspace{-6mm}
    	\includegraphics[width=3.6cm]{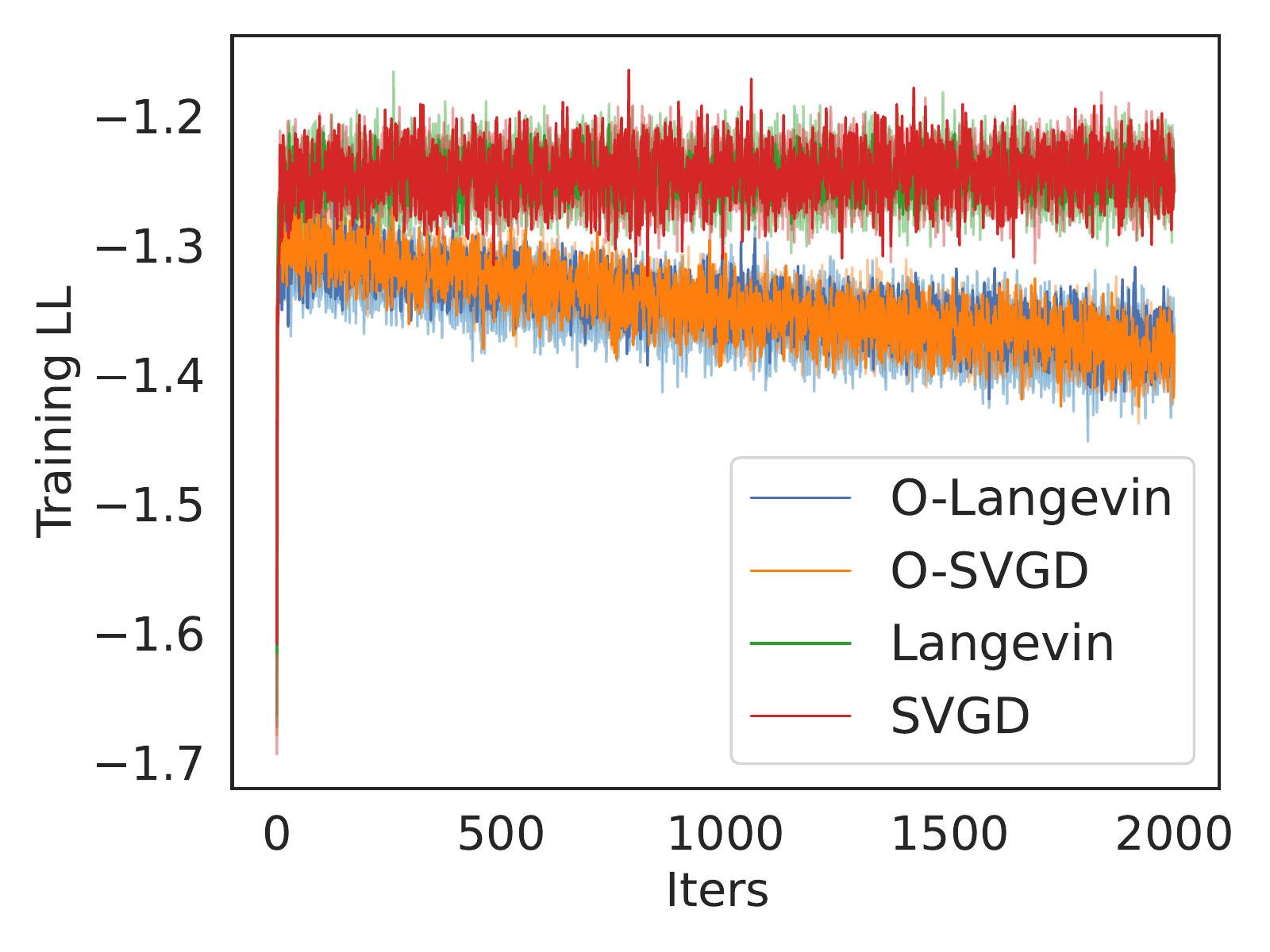}  &
    		\hspace{-6mm}
    	\includegraphics[width=3.6cm]{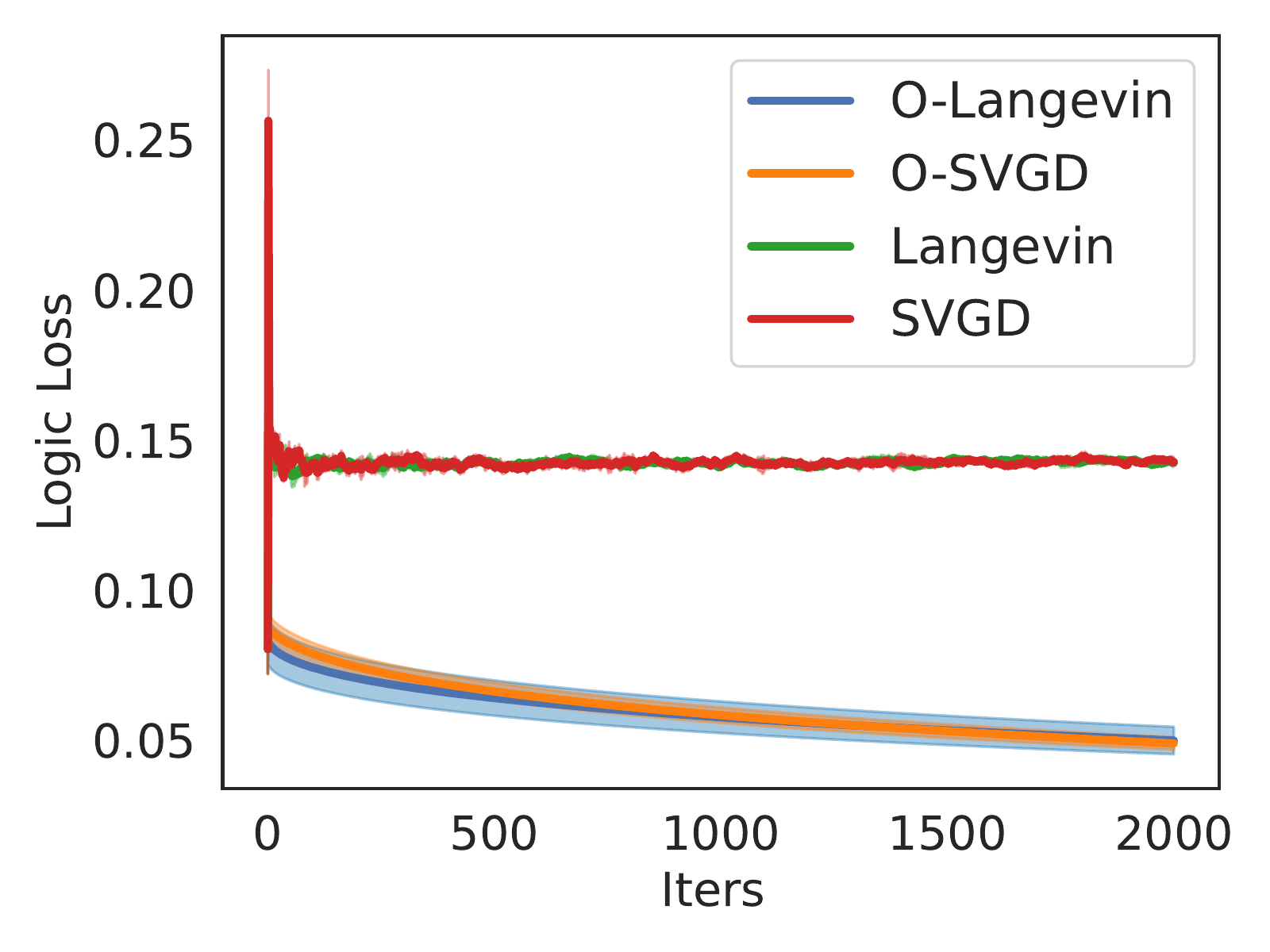} &
    	\hspace{-6mm}
    	\includegraphics[width=3.6cm]{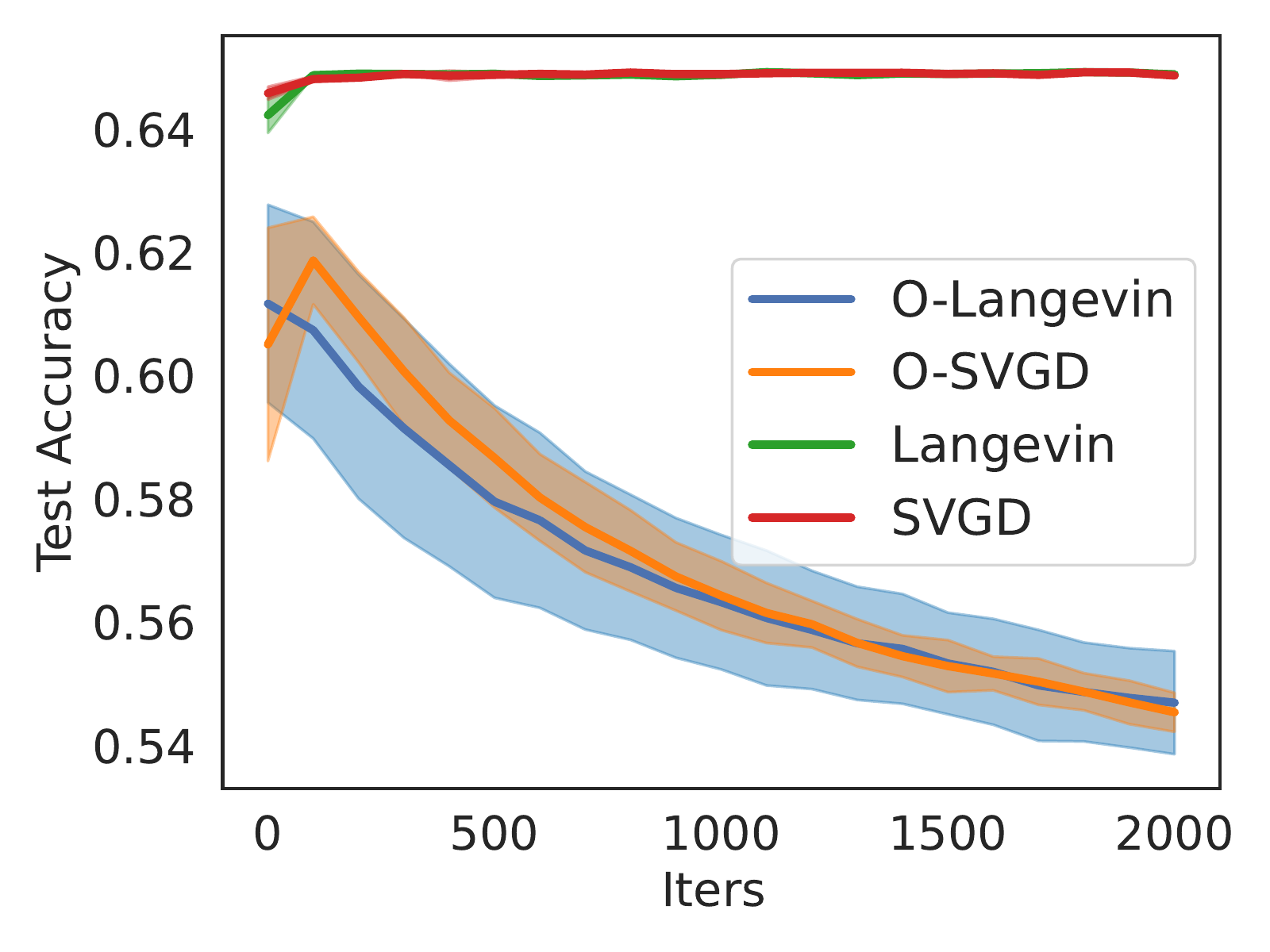}&
    	\hspace{-6mm}
    	\includegraphics[width=3.6cm]{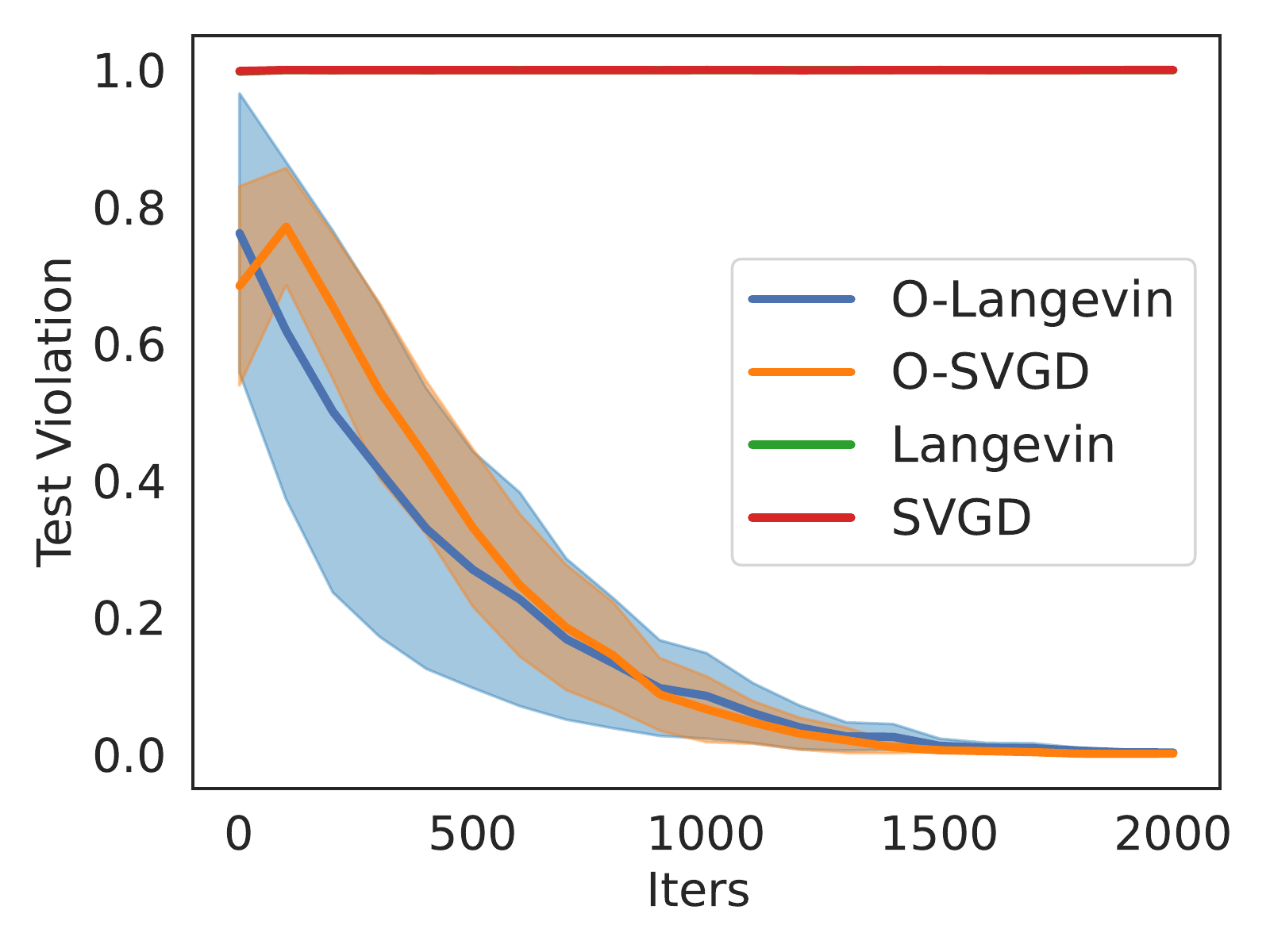}\\
    	\multicolumn{2}{c}{(a) Training Curve} &  \multicolumn{2}{c}{(b) Testing Curve}
    \end{tabular}
    \caption{Loan Classification: O-Langevin and O-SVGD satisfy the constraint well whereas standard Langevin and SVGD violate the constraint severely. 
    }
    \label{fig:logic-rule}
\end{figure}

\paragraph{Loan Classification with Logic Rules}\label{sec:exp:loan}

Following previous work~\citep{liu2021sampling}, we predict whether to lend loans to applicants and encode two logic rules: (1) an applicant must be rejected if not employed and having the lowest credit rank; (2) an applicant must be approved if employed over 15 years and having the highest credit rank. We use a Bayesian logistic regression model and set $g(x) = l_{\text{logic}}(x) = \mathbb{E}_{(\mu,\nu)\sim \mathcal{D}_{\text{logic}}}[\text{Loss}(\nu,\hat{\nu}(\mu;x))]$ where $\mathcal{D}_{\text{logic}}$ is a uniform distribution over all datapoints $(\mu,\nu)$ that satisfy the constraint, Loss refers to the classification loss and $\hat{\nu}$ is the prediction made by the model. During testing, we compute the violation which is the percentage of predictions that do not follow the logic rules.
In Figure~\ref{fig:logic-rule}, we see that both of our methods satisfy the logic rules well whereas standard algorithms violate the constraint severely. A drop in the accuracy of our methods might be caused by conflicts between the data and the logic rules, that is, the solution with high accuracy does not satisfy the constraint. A similar phenomenon has also been observed in ~\citet{liu2021sampling} which implies that this negative impact on the accuracy is most likely due to the task rather than the algorithm. Nevertheless, our methods can obtain models satisfying the constraint, thus guarantee their safety and interpretability.

\paragraph{Prior-Agnostic Bayesian Neural Networks}
Bayesian neural network (BNN) has been widely used in deep learning (DL) for quantifying uncertainty. The posterior of BNNs is determined by the prior which reflects our prior knowledge, and the likelihood which quantifies the data fitness. Due to complex NN architectures, specifying an appropriate prior for BNNs has been shown to be difficult~\citep{fortuin2021bayesian}. A poor prior often leads to a bad posterior and thus bad inference results. To automatically control the influence of the prior, one approach is to sample from the posterior with the constraint of a reasonably high data fitness.
We apply our methods to image classification on CIFAR10 with ResNet-18. Since the training loss is typically very small in DL, we set the constraint to be $g(x)=\text{Loss}(x)$. Similar to stochastic Langevin dynamics (SGLD)~\citep{welling2011bayesian} and SVGD, our methods are easy to be combined with stochastic gradients which are used on this task. We also ignore the second-order derivative terms in our methods for speedup. Besides standard algorithms, we also compare with a tempered SGLD which has been used in the literature for good performance~\citep{zhang2019cyclical}. We report test error, expected calibration error (ECE) and AUROC on the SVHN dataset (measuring out-of-distribution detection). 

From Table~\ref{tab:oe}, we see that O-Langevin and O-SVGD improve significantly over their unconstrained counterparts on all three metrics. O-Langevin has the lowest test error and the highest AUROC while O-SVGD has the lowest ECE. These results indicate that the constraint can automatically limit the effect of the prior, leading to a better posterior and thus better generalization and calibration.

\begin{table}[t]
    \caption{Results on Image Classification (\%). O-Langevin and O-SVGD significantly outperform unconstrained methods in terms of both generalization accuracy and calibration.}
  \label{tab:oe}
  \centering
  \begin{tabular}{lcccccc}
    \toprule & Test Error ($\downarrow$) &ECE   ($\downarrow$) &AUROC ($\uparrow$) \\
    \midrule
      SGLD &15.00 & 2.21   &89.41\\
      Tempered SGLD &4.73   &0.83 & 97.63\\
      O-Langevin  & {\bf 4.46}  &{ 0.87} & {\bf 98.68}  \\
      \hdashline
      SVGD &6.11    &0.93   &93.55\\
      O-SVGD  &4.92  &{\bf 0.77}   &94.69\\
    \bottomrule
  \end{tabular}
 
\end{table}

\section{Conclusion and Limitations}\label{sec:conclusion}

We propose a new variational framework with a designed gradient flow (O-Gradient) for sampling in implicitly defined constrained domains. O-Gradient is formed by two orthogonal directions where one drives the sampler towards the domain and the other explores the domain by decreasing a KL divergence. We prove the convergence of O-Gradient and apply the framework to both Langevin dynamics and SVGD. We empirically demonstrate the power of our methods on a variety of ML tasks. 

While ML achieves impressive performance, we must take realistic constraints into consideration when deploying it in daily life. This work provides a principled framework for solving constrained sampling problem with theoretical guarantees and practical algorithms.

One possible limitation of our methods is that the obtained samples are not exactly on the manifold due to working in an ambient space, though can be made arbitrarily close to the manifold by tuning hyperapameters. This might be solved by including a projection step. Another limitation is that we only consider an equality constraint. It will be interesting to incorporating inequality constraints into the framework. Besides, our theory only considers the convergence in the continuous-time limit. Future work could provide convergence analysis in discrete time, which may characterize the properties of practical algorithms more accurately.

\section*{Acknowledgements}
QL is supported by CAREER-1846421, SenSE2037267, EAGER-2041327, Office of Navy Research, and NSF AI Institute for Foundations of Machine Learning (IFML). XT is supported by the Singapore Ministry of Education (MOE) grant R-146-000-292-114.
\bibliography{zref} 

\begin{thebibliography}{39}
\providecommand{\natexlab}[1]{#1}
\providecommand{\url}[1]{\texttt{#1}}
\expandafter\ifx\csname urlstyle\endcsname\relax
  \providecommand{\doi}[1]{doi: #1}\else
  \providecommand{\doi}{doi: \begingroup \urlstyle{rm}\Url}\fi

\bibitem[Brooks et~al.(2011)Brooks, Gelman, Jones, and
  Meng]{brooks2011handbook}
Steve Brooks, Andrew Gelman, Galin Jones, and Xiao-Li Meng.
\newblock \emph{Handbook of markov chain monte carlo}.
\newblock CRC press, 2011.

\bibitem[Brubaker et~al.(2012)Brubaker, Salzmann, and
  Urtasun]{brubaker2012family}
Marcus Brubaker, Mathieu Salzmann, and Raquel Urtasun.
\newblock A family of mcmc methods on implicitly defined manifolds.
\newblock In \emph{Artificial intelligence and statistics}, pages 161--172.
  PMLR, 2012.

\bibitem[Byrne and Girolami(2013)]{byrne2013geodesic}
Simon Byrne and Mark Girolami.
\newblock Geodesic monte carlo on embedded manifolds.
\newblock \emph{Scandinavian Journal of Statistics}, 40\penalty0 (4):\penalty0
  825--845, 2013.

\bibitem[Calders and Verwer(2010)]{calders2010three}
Toon Calders and Sicco Verwer.
\newblock Three naive bayes approaches for discrimination-free classification.
\newblock \emph{Data Mining and Knowledge Discovery}, 21\penalty0 (2):\penalty0
  277--292, 2010.

\bibitem[Chang and Pollard(1997)]{chang1997conditioning}
Joseph~T Chang and David Pollard.
\newblock Conditioning as disintegration.
\newblock \emph{Statistica Neerlandica}, 51\penalty0 (3):\penalty0 287--317,
  1997.

\bibitem[Diaconis et~al.(2013)Diaconis, Holmes, and
  Shahshahani]{diaconis2013sampling}
Persi Diaconis, Susan Holmes, and Mehrdad Shahshahani.
\newblock Sampling from a manifold.
\newblock In \emph{Advances in modern statistical theory and applications: a
  Festschrift in honor of Morris L. Eaton}, pages 102--125. Institute of
  Mathematical Statistics, 2013.

\bibitem[Duncan et~al.(2019)Duncan, N{\"u}sken, and
  Szpruch]{duncan2019geometry}
Andrew Duncan, Nikolas N{\"u}sken, and Lukasz Szpruch.
\newblock On the geometry of stein variational gradient descent.
\newblock \emph{arXiv preprint arXiv:1912.00894}, 2019.

\bibitem[Fortuin et~al.(2021)Fortuin, Garriga-Alonso, Wenzel, R{\"a}tsch,
  Turner, van~der Wilk, and Aitchison]{fortuin2021bayesian}
Vincent Fortuin, Adri{\`a} Garriga-Alonso, Florian Wenzel, Gunnar R{\"a}tsch,
  Richard Turner, Mark van~der Wilk, and Laurence Aitchison.
\newblock Bayesian neural network priors revisited.
\newblock \emph{arXiv preprint arXiv:2102.06571}, 2021.

\bibitem[Girolami and Calderhead(2011)]{girolami2011riemann}
Mark Girolami and Ben Calderhead.
\newblock Riemann manifold langevin and hamiltonian monte carlo methods.
\newblock \emph{Journal of the Royal Statistical Society: Series B (Statistical
  Methodology)}, 73\penalty0 (2):\penalty0 123--214, 2011.

\bibitem[Gorham and Mackey(2017)]{gorham2017measuring}
Jackson Gorham and Lester Mackey.
\newblock Measuring sample quality with kernels.
\newblock In \emph{International Conference on Machine Learning}, pages
  1292--1301. PMLR, 2017.

\bibitem[Gritsenko et~al.(2020)Gritsenko, Salimans, van~den Berg, Snoek, and
  Kalchbrenner]{gritsenko2020spectral}
Alexey Gritsenko, Tim Salimans, Rianne van~den Berg, Jasper Snoek, and Nal
  Kalchbrenner.
\newblock A spectral energy distance for parallel speech synthesis.
\newblock \emph{Advances in Neural Information Processing Systems},
  33:\penalty0 13062--13072, 2020.

\bibitem[Hebey(2000)]{hebey2000nonlinear}
Emmanuel Hebey.
\newblock \emph{Nonlinear Analysis on Manifolds: Sobolev Spaces and
  Inequalities: Sobolev Spaces and Inequalities}, volume~5.
\newblock American Mathematical Soc., 2000.

\bibitem[Kohavi(1996)]{kohavi1996scaling}
Ron Kohavi.
\newblock Scaling up the accuracy of naive-bayes classifiers: A decision-tree
  hybrid.
\newblock In \emph{Kdd}, volume~96, pages 202--207, 1996.

\bibitem[Kook et~al.(2022)Kook, Lee, Shen, and Vempala]{kook2022sampling}
Yunbum Kook, Yin~Tat Lee, Ruoqi Shen, and Santosh~S Vempala.
\newblock Sampling with riemannian hamiltonian monte carlo in a constrained
  space.
\newblock \emph{arXiv preprint arXiv:2202.01908}, 2022.

\bibitem[LeCun et~al.(2006)LeCun, Chopra, Hadsell, Ranzato, and
  Huang]{lecun2006tutorial}
Yann LeCun, Sumit Chopra, Raia Hadsell, M~Ranzato, and F~Huang.
\newblock A tutorial on energy-based learning.
\newblock \emph{Predicting structured data}, 1, 2006.

\bibitem[Lee and Vempala(2018)]{lee2018convergence}
Yin~Tat Lee and Santosh~S Vempala.
\newblock Convergence rate of riemannian hamiltonian monte carlo and faster
  polytope volume computation.
\newblock In \emph{Proceedings of the 50th Annual ACM SIGACT Symposium on
  Theory of Computing}, pages 1115--1121, 2018.

\bibitem[Lelievre and Stoltz(2016)]{lelievre2016partial}
Tony Lelievre and Gabriel Stoltz.
\newblock Partial differential equations and stochastic methods in molecular
  dynamics.
\newblock \emph{Acta Numerica}, 25:\penalty0 681--880, 2016.

\bibitem[Lelievre et~al.(2012)Lelievre, Rousset, and
  Stoltz]{lelievre2012langevin}
Tony Lelievre, Mathias Rousset, and Gabriel Stoltz.
\newblock Langevin dynamics with constraints and computation of free energy
  differences.
\newblock \emph{Mathematics of computation}, 81\penalty0 (280):\penalty0
  2071--2125, 2012.

\bibitem[Leli{\`e}vre et~al.(2019)Leli{\`e}vre, Rousset, and
  Stoltz]{lelievre2019hybrid}
Tony Leli{\`e}vre, Mathias Rousset, and Gabriel Stoltz.
\newblock Hybrid monte carlo methods for sampling probability measures on
  submanifolds.
\newblock \emph{Numerische Mathematik}, 143\penalty0 (2):\penalty0 379--421,
  2019.

\bibitem[Leli{\`e}vre et~al.(2022)Leli{\`e}vre, Stoltz, and
  Zhang]{lelievre2022multiple}
Tony Leli{\`e}vre, Gabriel Stoltz, and Wei Zhang.
\newblock Multiple projection markov chain monte carlo algorithms on
  submanifolds.
\newblock \emph{IMA Journal of Numerical Analysis}, 2022.

\bibitem[Li and Wang(2006)]{li2006weighted}
Peter Li and Jiaping Wang.
\newblock Weighted poincar{\'e} inequality and rigidity of complete manifolds.
\newblock In \emph{Annales scientifiques de l'Ecole normale sup{\'e}rieure},
  volume~39, pages 921--982, 2006.

\bibitem[Liu(2017)]{liu2017stein}
Qiang Liu.
\newblock Stein variational gradient descent as gradient flow.
\newblock \emph{Advances in Neural Information Processing Systems (NeurIPS)},
  2017.

\bibitem[Liu and Wang(2016)]{liu2016stein}
Qiang Liu and Dilin Wang.
\newblock Stein variational gradient descent: A general purpose {Bayesian}
  inference algorithm.
\newblock In \emph{Advances in Neural Information Processing Systems
  (NeurIPS)}, 2016.

\bibitem[Liu and Vicente(2020)]{liu2020accuracy}
Suyun Liu and Luis~Nunes Vicente.
\newblock Accuracy and fairness trade-offs in machine learning: A stochastic
  multi-objective approach.
\newblock \emph{arXiv preprint arXiv:2008.01132}, 2020.

\bibitem[Liu et~al.(2021)Liu, Tong, and Liu]{liu2021sampling}
Xingchao Liu, Xin Tong, and Qiang Liu.
\newblock Sampling with trusthworthy constraints: A variational gradient
  framework.
\newblock \emph{Advances in Neural Information Processing Systems}, 34, 2021.

\bibitem[Martinez et~al.(2020)Martinez, Bertran, and
  Sapiro]{martinez2020minimax}
Natalia Martinez, Martin Bertran, and Guillermo Sapiro.
\newblock Minimax pareto fairness: A multi objective perspective.
\newblock In \emph{International Conference on Machine Learning}, pages
  6755--6764. PMLR, 2020.

\bibitem[Meyer et~al.(2011)Meyer, Bonnabel, and Sepulchre]{meyer2011linear}
Gilles Meyer, Silvere Bonnabel, and Rodolphe Sepulchre.
\newblock Linear regression under fixed-rank constraints: a riemannian
  approach.
\newblock In \emph{Proceedings of the 28th international conference on machine
  learning}, 2011.

\bibitem[Patterson and Teh(2013)]{patterson2013stochastic}
Sam Patterson and Yee~Whye Teh.
\newblock Stochastic gradient riemannian langevin dynamics on the probability
  simplex.
\newblock In \emph{Advances in Neural Information Processing Systems
  (NeurIPS)}, 2013.

\bibitem[Sejdinovic et~al.(2012)Sejdinovic, Gretton, Sriperumbudur, and
  Fukumizu]{sejdinovic2012hypothesis}
Dino Sejdinovic, Arthur Gretton, Bharath Sriperumbudur, and Kenji Fukumizu.
\newblock Hypothesis testing using pairwise distances and associated kernels
  (with appendix).
\newblock \emph{International conference on machine learning}, 2012.

\bibitem[Sejdinovic et~al.(2013)Sejdinovic, Sriperumbudur, Gretton, and
  Fukumizu]{sejdinovic2013equivalence}
Dino Sejdinovic, Bharath Sriperumbudur, Arthur Gretton, and Kenji Fukumizu.
\newblock Equivalence of distance-based and rkhs-based statistics in hypothesis
  testing.
\newblock \emph{The annals of statistics}, pages 2263--2291, 2013.

\bibitem[Sharma and Zhang(2021)]{sharma2021nonreversible}
Upanshu Sharma and Wei Zhang.
\newblock Nonreversible sampling schemes on submanifolds.
\newblock \emph{SIAM Journal on Numerical Analysis}, 59\penalty0 (6):\penalty0
  2989--3031, 2021.

\bibitem[Simon(2014)]{simon2014introduction}
Leon Simon.
\newblock Introduction to geometric measure theory.
\newblock \emph{Tsinghua Lectures}, 2014.

\bibitem[Song and Ermon(2019)]{song2019generative}
Yang Song and Stefano Ermon.
\newblock Generative modeling by estimating gradients of the data distribution.
\newblock In \emph{Advances in Neural Information Processing Systems}, pages
  11895--11907, 2019.

\bibitem[Stoltz et~al.(2010)Stoltz, Rousset, et~al.]{stoltz2010free}
Gabriel Stoltz, Mathias Rousset, et~al.
\newblock \emph{Free energy computations: A mathematical perspective}.
\newblock World Scientific, 2010.

\bibitem[Wang et~al.(2020)Wang, Lei, and Panageas]{wang2020fast}
Xiao Wang, Qi~Lei, and Ioannis Panageas.
\newblock Fast convergence of langevin dynamics on manifold: Geodesics meet
  log-sobolev.
\newblock \emph{Advances in Neural Information Processing Systems},
  33:\penalty0 18894--18904, 2020.

\bibitem[Welling and Teh(2011)]{welling2011bayesian}
Max Welling and Yee~W Teh.
\newblock {{Bayesian}} learning via stochastic gradient {Langevin} dynamics.
\newblock In \emph{International Conference on Machine Learning}, 2011.

\bibitem[Zappa et~al.(2018)Zappa, Holmes-Cerfon, and Goodman]{zappa2018monte}
Emilio Zappa, Miranda Holmes-Cerfon, and Jonathan Goodman.
\newblock Monte carlo on manifolds: sampling densities and integrating
  functions.
\newblock \emph{Communications on Pure and Applied Mathematics}, 71\penalty0
  (12):\penalty0 2609--2647, 2018.

\bibitem[Zhang et~al.(2020)Zhang, Li, Zhang, Chen, and
  Wilson]{zhang2019cyclical}
Ruqi Zhang, Chunyuan Li, Jianyi Zhang, Changyou Chen, and Andrew~Gordon Wilson.
\newblock Cyclical stochastic gradient mcmc for bayesian deep learning.
\newblock \emph{International Conference on Learning Representations}, 2020.

\bibitem[Zhang(2020)]{zhang2020ergodic}
Wei Zhang.
\newblock Ergodic sdes on submanifolds and related numerical sampling schemes.
\newblock \emph{ESAIM: Mathematical Modelling and Numerical Analysis},
  54\penalty0 (2):\penalty0 391--430, 2020.

\end{thebibliography}
\bibliographystyle{plainnat}

\section*{Checklist}


\begin{enumerate}

\item For all authors...
\begin{enumerate}
  \item Do the main claims made in the abstract and introduction accurately reflect the paper's contributions and scope?
    \answerYes{}
  \item Did you describe the limitations of your work?
    \answerYes{}
  \item Did you discuss any potential negative societal impacts of your work?
    \answerNA{}
  \item Have you read the ethics review guidelines and ensured that your paper conforms to them?
    \answerYes{}
\end{enumerate}

\item If you are including theoretical results...
\begin{enumerate}
  \item Did you state the full set of assumptions of all theoretical results?
    \answerYes{}
        \item Did you include complete proofs of all theoretical results?
    \answerYes{}
\end{enumerate}

\item If you ran experiments...
\begin{enumerate}
  \item Did you include the code, data, and instructions needed to reproduce the main experimental results (either in the supplemental material or as a URL)?
    \answerYes{}
  \item Did you specify all the training details (e.g., data splits, hyperparameters, how they were chosen)?
    \answerYes{}
        \item Did you report error bars (e.g., with respect to the random seed after running experiments multiple times)?
    \answerYes{}
        \item Did you include the total amount of compute and the type of resources used (e.g., type of GPUs, internal cluster, or cloud provider)?
    \answerYes{}
\end{enumerate}

\item If you are using existing assets (e.g., code, data, models) or curating/releasing new assets...
\begin{enumerate}
  \item If your work uses existing assets, did you cite the creators?
    \answerNA{}
  \item Did you mention the license of the assets?
    \answerNA{}
  \item Did you include any new assets either in the supplemental material or as a URL?
    \answerNA{}
  \item Did you discuss whether and how consent was obtained from people whose data you're using/curating?
    \answerNA{}
  \item Did you discuss whether the data you are using/curating contains personally identifiable information or offensive content?
    \answerNA{}
\end{enumerate}

\item If you used crowdsourcing or conducted research with human subjects...
\begin{enumerate}
  \item Did you include the full text of instructions given to participants and screenshots, if applicable?
    \answerNA{}
  \item Did you describe any potential participant risks, with links to Institutional Review Board (IRB) approvals, if applicable?
    \answerNA{}
  \item Did you include the estimated hourly wage paid to participants and the total amount spent on participant compensation?
    \answerNA{}
\end{enumerate}

\end{enumerate}


\appendix

\section{proofs}
\begin{proof}[Proof of Lemma \ref{lem:RKHSform}]
When $\calH$ is the RKHS with kernel $k$
\bb
&\E_{q_t}[(D(s_\pi - s_{q_t})) \tt u ] - \frac{1}{2} \norm{D u}_{\H}^2\\
&=\E_{q_t}[(D(s_\pi - s_{q_t})) \tt Du ] - \frac{1}{2} \norm{D u}_{\H}^2\\
&=\langle R_{D(s_\pi-s_{q_t})}, D u\rangle_{\H}-\frac{1}{2} \norm{D u}_{\H}^2
\ee
where $R_v$ is the Riesz representation of the linear function $u\to \E_{q_t} v\tt u$.  
Therefore a solution to \eqref{eq:proju} is
\[
Du=DR_{D(s_\pi-s_{q_t})}
\]
which can be written as 
\bb
Du(x)&=\int D(x)k(x,y)D(y) (s_\pi(y)-s_q(y)) q_t(y)dy\\
&=\int k_\bot(x,y) (s_\pi(y)-s_q(y)) q_t(y)dy\\
&= \int (k_\bot(x,y) s_\pi(y)+\nabla_y k_\bot (x,y))q_t(y)dy\\
&=\E_{y\sim q_t}(k_\bot(x,y)s_\pi(y)+\nabla_y k_\bot(x,y))
\ee

\end{proof}

\begin{proof}[Proof of Theorem \ref{thm:LD}]
First, we show the components of $r$ can be written as
\begin{align*}
r_i(x)=\sum_{j}\partial_j D_{i,j}(x)&=-\frac{\sum_j \partial^2_{i,j} g \partial_j g+\partial_i g\partial_{j,j}g}{\|\nabla g\|^2}+\frac{2(\nabla g^T \nabla^2 g\nabla g)\partial_i g}{\|\nabla g\|^4}.
\end{align*}
To this end, note that 
\[
D_{i,j}(x)=1-\frac{\partial_i g\partial_j g}{\|\nabla g\|^2}
\]
Then by product rule, and the fact that $\partial_j \|\nabla g\|^2=2\sum_k \partial^2_{k,j} g \partial_k g=2[\nabla^2 g \nabla g]_j$
\[
\partial_j D_{i,j}(x)=\frac{-\partial^2_{i,j} g\partial_j g-\partial_i g\partial^2_{jj} g}{\|\nabla g\|^2}+\frac{2\partial_i g\partial_j g [\nabla^2 g\nabla g]_j}{\|\nabla g\|^4}.
\]
So 
\[
\sum_j\partial_j D_{i,j}(x)=\frac{-\sum_j\partial^2_{i,j} g\partial_j g-\sum_j\partial_i g\partial^2_{jj} g}{\|\nabla g\|^2}+\frac{2\partial_i g(\nabla g\tt \nabla^2 g\nabla g)}{\|\nabla g\|^4}.
\]
Comparing with the each component of 
\[
r=-\frac{\nabla^2 g\nabla g}{\|\nabla g\|^2}-\frac{\text{tr}(\nabla^2 g)}{\|\nabla g\|^2}\nabla g+\frac{2\nabla g^T \nabla^2 g\nabla g}{\|\nabla g\|^4}\nabla g,
\]
it is easy to see they are the same.

Next, recall the O-gradient density flow \eqref{eq:densityLD}
\begin{align*}
\frac{d}{dt} q
&=-\nabla \cdot (\phi(x) q(x))+\nabla \cdot ( D(x)\nabla q(x))\\
&=-\nabla \cdot (\phi(x) q(x))+\sum_{i,j}\partial_i (D_{i,j}(x) \partial_jq(x))\\
&=-\nabla \cdot (\phi(x) q(x))+\sum_{i,j}\partial_i D_{i,j}(x) \partial_jq(x)+\sum_{i,j} D_{i,j}(x) \partial^2_{i,j}q(x)\\
&=:-\nabla \cdot (\phi(x) q(x))+(I). 
\end{align*}
Meanwhile, the FPE of the SDE follows 
\begin{align}
\label{tmp:FPE}
\frac{d}{dt}q=-\nabla \cdot (\phi(x) q(x))-\nabla \cdot (r(x) q(x))+\sum_{i,j} \partial_{i,j}^2 (D_{i,j} (x)q(x))
\end{align}
Then note that
\begin{align*}
(II):&=-\nabla \cdot (r(x) q(x))=-\sum_i\partial_i\left(\sum_{j}\partial_j D_{i,j}(x) q(x)\right)\\
&=-\sum_{i,j}\partial_{i,j}^2 D_{i,j}(x) q(x)-\sum_{i,j}\partial_j D_{i,j}(x) \partial_iq(x).
\end{align*}
Also note that 
\begin{align*}
\sum_{i,j} \partial_{i,j}^2 (D_{i,j} (x)q(x))&= \sum_{i,j} \partial_{i,j}^2 D_{i,j} (x)q(x)+\sum_{i,j} \partial_{j} D_{i,j} (x)\partial_i q(x)\\
&\quad+  \sum_{i,j}  \partial_i D_{i,j} (x)\partial_j q(x)+\sum_{i,j} D_{i,j} (x)\partial^2_{i,j} q(x)\\
&=(I)-(II). 
\end{align*}
So we arrive at our first claim.

For the second claim, using $\nabla g(x)^T D(x)=0$,  the Ito formula indicates that 
\begin{align*}
d g(X_t)&= \nabla g(X_t)^T[D(X_t)\nabla \log \pi(X_t)-\frac{c g(X_t)}{\|\nabla g(X_t)\|^2}\nabla g(X_t)+r(X_t)+D(X_t)\sqrt{2}dW_t]\\
&\quad + \text{tr}(D(X_t)\nabla^2g(X_t) D(X_t))dt\\
&=-c g(X_t) dt+ \nabla g(X_t)^Tr(X_t)+\text{tr}(D(X_t)\nabla^2g(X_t) D(X_t))dt.
\end{align*}
So it suffices to show that 
\[
\nabla g(x)^Tr(x)+\text{tr}(D(x)\nabla^2g(x) D(x))=0.
\]
To continue, we suppress the expression of $x$ in below to keep formulas short. 
We first note that $D_{i,j}=1_{i=j}-\partial_i g\partial_j g/\|\nabla g\|^2$, 
\[
\partial_j D_{i,j}=\frac{-\partial_{i,j} g\partial_j g-\partial_{i} g\partial_{j,j} g}{\|\nabla g\|^2}+\frac{2\sum_{k}\partial_i g\partial_j g \partial_{j,k}^2 g \partial_k g}{\|\nabla g\|^4}
\]
We plug this into the computation of $\nabla g^T r=\sum_i \partial_i g\partial_j D_{i,j}$. We note that 
\[
\sum_{i,j}-\frac{\partial_i g\partial_{i,j} g\partial_j g}{\|\nabla g\|^2}=\frac{-(\nabla g)^T\nabla^2 g \nabla g }{\|\nabla g\|^2}.
\]
\[
\sum_{i,j}-\frac{(\partial_i g)^2\partial_{j,j} g\partial_j g}{\|\nabla g\|^2}=-\text{tr}(\nabla^2 g). 
\]
\begin{align*}
\frac{2\sum_{i,j}\sum_{k}\partial_i g\partial_i g\partial_j g \partial_{j,k}^2 g \partial_k g}{\|\nabla g\|^4}&=
\frac{\sum_{j,k}\|\nabla g\|^2\partial_j g \partial_{j,k}^2 g \partial_k g}{\|\nabla g\|^4}=\frac{2(\nabla g)^T\nabla^2 g \nabla g }{\|\nabla g\|^2}
\end{align*}
Therefore
\begin{align*}
\sum_i \partial_i g\partial_j D_{i,j}&=\frac{(\nabla g)^T\nabla^2 g \nabla g }{\|\nabla g\|^2}-\text{tr}(\nabla^2 g)\\
&=\text{tr}\left(-\nabla^2 g+\nabla^2 g\frac{ \nabla g (\nabla g)^T}{\|\nabla g\|^2}\right)\\
&=\text{tr}\left(-\nabla^2 gD\right)=\text{tr}\left(-D\nabla^2 gD\right),\quad \mbox{since $D^2=D$}.
\end{align*}
This completes our proof. 

For the last claim, we note that 
\begin{align*}
\mathcal{L} f&=\nabla f^T [D\nabla \log \pi-\frac{\psi(g) }{\|\nabla g\|^2}\nabla g+r] +\text{tr}(D\nabla^2 fD)\\
&=\nabla f^T [\nabla \log \pi-\frac{ \nabla^2 g}{\|\nabla g\|^2}\nabla g] +\text{tr}(D\nabla^2 f)\\
&=\nabla f^T \nabla \log \pi+\text{tr}(\nabla^2 f)-\frac{ \nabla f^T\nabla^2 g\nabla g}{\|\nabla g\|^2} +\frac{\nabla g^T \nabla^2 f\nabla g}{\|\nabla g\|^2}\\
\end{align*}
Finally, we note that $\nabla f^T \nabla g=0$. Take derivative of this identity we find 
\[
\nabla^2 f \nabla g+\nabla^2 g \nabla f=0
\]
Therefore 
\[
\nabla f^T \nabla^2 g\nabla g=-\nabla g^T \nabla^2 f\nabla g
\]
and 
\[
\mathcal{L}f=\nabla f^T \nabla \log \pi+\text{tr}(\nabla^2 f)
\]
which is the same as the generator of the Langevin diffusion. 
\end{proof}

\begin{remark}
If we drop $r$ and implementing a naive SDE:
\[
dx_t=\phi(x)+\sqrt{2}D(x)dw_t,
\]
the FPE of this SDE will be
\[
\frac{d}{dt}q=-\nabla \cdot (\phi(x) q(x))+\sum_{i,j} \partial_{i,j}^2 (D_{i,j} (x)q(x)).
\]
It is identical to \eqref{tmp:FPE} but without the term $-\nabla \cdot (r(x) q(x))=(II)$. In other words, it will not match \eqref{eq:densityLD} unless $(II)=0$, which happens when $\sum_{i,j}\partial_{i,j}^2 D_{i,j}=\sum_{j}\partial_j D_{i,j}\equiv 0$. But it should be pointed out that if $g$ is an affine function, $D(x)$ is a constant matrix, then $(II)\equiv 0$, and $r$ is safe to be dropped out. 
\end{remark}

\begin{proof}[Proof of Proposition \ref{pro:disint}]
We will first show that for any function $f$, the following holds
\[
\int \pi^g(z) f(x) \pi_{\eta,z}(x) dxdz\to \int f(x)\pi(x)dx,
\]
when $\eta\to 0$. So \eqref{eq:condmeas} holds for $\Pi_z$ as the weak limit of $\pi_{\eta,z}$.

We note that 
\[
\pi_{\eta,z}(x)=\frac{\pi(x)\exp(-\frac1{2\eta}(g(x)-z)^2)}{\sqrt{2\eta\pi}Z_{\eta,z}}
\]
where the normalizing constant is given by
\[
Z_{\eta,z}=\frac{1}{\sqrt{2\pi\eta}}\int \pi(x) \exp(-\frac1{2\eta} (g(x)-z)^2)dx.
\]
Because $\pi^g$ is the density of $g(X)$, so for any function $h$:
\[
\int h(g(x)) \pi(x)dx=\E_{X\sim\pi}[ h(g(X))]=\int h(y) \pi^g(y)dy.
\]
We pick $h(y)=\exp(-\frac1{2\eta}(y-z)^2)$, we obtain that 
\begin{align*}
Z_{\eta,z}&=\int \frac{1}{\sqrt{2\pi\eta}}\exp(-\frac1{2\eta}(y-z)^2) \pi^g(y)dy\\
&=\E \pi^g(z+\sqrt{\eta}\xi)=\pi^g(z)(1+R(z))^{-1}.
\end{align*}
where the $|R|\leq 2L\sqrt{\eta}$ and $L$ is the regularity constant of $\pi^g$. Therefore 
\begin{align*}
    &\int \pi^g(z) f(x) \pi_{\eta,z}(x) dxdz\\
    &=\int \pi(x) f(x)\left(\int\frac1{\sqrt{2\pi\eta}}\exp(-\frac1{2\eta}(g(x)-z)^2)\frac{\pi^g(z)}{Z_{\eta,z}} dz\right)dx\\
    &=\int \pi(x) f(x)\left(\int\frac1{\sqrt{2\pi\eta}}\exp(-\frac1{2\eta}(g(x)-z)^2) dz\right)dx+\E R(g(X))\\
    &=\int \pi(x)f(x)dx+\E R(g(X)).
\end{align*}
Since $\E R(g(X))\leq 2L\sqrt{\eta}$, we find our first claim when $\eta\to 0$. 

Next note that if we pick $f(x)=1_{|g(x)-z|\geq \epsilon}$
\begin{align*}
    &\int f(x) \pi_{\eta,z}(x) dx\\
    &=\int 1_{|g(x)-z|\geq \epsilon} \frac1{\sqrt{2\pi\eta}}\exp(-\frac1{2\eta}(g(x)-z)^2)\frac{\pi(x)}{Z_{\eta,z}} dx\\
    &\leq \int  \frac1{\sqrt{2\pi\eta}}\exp(-\frac{\epsilon^2}{2\eta})\frac{\pi(x)}{Z_{\eta,z}} dx=
    \frac1{\sqrt{2\pi\eta}Z_{\eta,z}}\exp(-\frac{\epsilon^2}{2\eta}).
\end{align*}
When $\eta\to 0$, since $Z_{\eta,z}\to \pi^g(z)$,  and $\frac1{\sqrt{2\eta}}\exp(-\frac{\epsilon^2}{2\eta})\to 0$, we find that 
\[
\Pi_z(|g(X)-z|\geq \epsilon)=\lim_{\eta\to 0}\int f(x) \pi_{\eta,z}(x) dx\to 0
\]

For the Stein equation part, note that for each $\pi_{\eta,z}$, we have the following by Stein's identity: 
$$
\E_{\pi_{\eta,z}}[(\dd \log \pi(x) - \frac{1}{2\eta} (g(x)-z) \dd g(x) )\tt  \phi(x) + \dd \tt \phi(x)] = 0. 
$$
But $\dd g(x) \tt \phi(x) = 0$ for $\phi\in \calH_\bot$. This gives 
$$
\E_{\pi_{\eta,z}} [\stein_\pi \phi] = 0,~~~\forall \eta. 
$$
Taking $\eta\to0$ yields that $\E_{\pi_{z}}[\stein_\pi \phi] =0.$ 
\end{proof}

\begin{lem}
\label{lem:RN}
Suppose $q(x),q^g(z)$ and $\pi(x),\pi^g(z),\pi_z(x)$ are all $C^1$ functions, then the Radon--Nikodym derivative between $q_z$ and $\pi_z$  can be written as  
\[
\frac{d q_z}{d\pi_z}(x)=\frac{\pi^g(z)q(x)}{q^g(z) \pi(x)},\quad z=g(x).
\]
In particular,
\[
D (s_{q_z}(x)-s_{\pi_z}(x))=D(s_q(x)-s_\pi(x)). 
\]
\end{lem}
\begin{proof}[Proof of Lemma \ref{lem:RN}]
We will show \eqref{eq:condmeas} holds, where  $dq_z=\frac{dq_z}{d\pi_z}d\pi_z$ with our choice of RN derivative. 
This can be done using
\begin{align*}
\E_{z\sim q^g}\E_{x\sim q_z}[f(x)]
&=\int_R dz q^g(z) \int_{\calG_z} f dq_z\\
&=\int_R dz q^g(z) \int_{\calG_z} \frac{\pi^g(z) q}{q^g(z) \pi} f d\pi_z \\
&=\int_R \pi^g(z) dz  \int_{\calG_z} \frac{q}{\pi} f d\pi_z\\
&=\E_{z\sim \pi^g}\E_{y\sim \pi_z}[f(x) q(x)/\pi(x)]\\
&=\E_{x\sim \pi} [f(x) q(x)/\pi(x)]=\E_{x\sim q} [f(x)].
\end{align*}
Moreover
\[
D (s_{q_z}(x)-s_{\pi_z}(x))=D(s_q(x)-s_\pi(x))+
D(\nabla g(x)s_{q^g}(g(x))-\nabla g(x)s_{q^g}(g(x)))
=D(s_q(x)-s_\pi(x)).
\]
\end{proof}

\begin{proof}[Proof of Proposition \ref{pro:meandiff}]
Note that 
\begin{align*}
\left|\E_{\Pi_0} [f]-\int^{\delta}_{-\delta} q_g(z)\E_{\Pi_z} [f] dz\right|
&\leq \int^{\delta}_{-\delta} q_g(z)\left|\E_{\Pi_0} [f]-\E_{\Pi_z} [f]\right| dz\\
&\leq \max_{|z|\leq \delta} |\E_{\Pi_z} [f]-\E_{\Pi_0} [f]|.
\end{align*}
Meanwhile
\begin{align*}
\E_q [f]-\int^{\delta}_{-\delta} q_g(z)\E_{\Pi_z} [f] dz
=\int^{\delta}_{-\delta} q_g(z)\left(\E_{ q_z} [f]-\E_{\Pi_z} [f]\right) dz.
\end{align*}
Then  we note the following holds when we restrict $f$ on $\calG_z$. We use a paramaterization of $\calG_z$ with dummy variable $y$ and the inherited metric form $\RR^d$, 
\begin{align*}
    (\E_{q_z} [f]-\E_{\Pi_z} [f])^2
    &=\left(\int_{\calG_z}  \Pi_z(y)(\frac{q_z(y)}{\Pi_z(y)}-1)f(y)dy\right)^2\\
    &\leq \int_{\calG_z}  \Pi_z(y)(\sqrt{\frac{q_z(y)}{\Pi_z(y)}}-1)^2dy \cdot\int_{\calG_z}  \Pi_z(y)(\sqrt{\frac{q_z(y)}{\Pi_z(y)}}+1)^2f(y)^2dy\\
    &\leq 2\int_{\calG_z}  \Pi_z(y)(\sqrt{\frac{q_z(y)}{\Pi_z(y)}}-1)^2dy\cdot \int_{\calG_z}  \Pi_z(y)(\frac{q_z(y)}{\Pi_z(y)}+1)dy\\
    &=4\int_{\calG_z}  \Pi_z(y)(\sqrt{\frac{q_z(y)}{\Pi_z(y)}}-1)^2dy\\
    &=8  (1-\E_{\Pi_z}\sqrt{\frac{q_z(y)}{\Pi_z(y)}})=8(1-b)\mbox{ with }b=\E_{\Pi_z}\sqrt{\frac{q_z(y)}{\Pi_z(y)}}\leq 1.
    \end{align*}
Then note that 
\[
\var_{\Pi_z}\sqrt{\frac{q_z(y)}{\Pi_z(y)}}=
\E_{\Pi_z}[\frac{q_z(y)}{\Pi_z(y)}]-b^2=1-b^2.
\]
  
Therefore by $\kappa$-PI, we have
    \begin{align*}
    1-b\leq (1-b^2)&\leq \kappa\int_{\calG_z}  \Pi_z(y)\left\|\nabla \sqrt{\frac{q_z(y)}{\Pi_z(y)}}\right\|_{\calG_z}^2dy\\
    &=\kappa\int_{\calG_z}  q_z(y)\|(s_{q_z}(y)-s_{\Pi_z}(y))\|_{\calG_z}^2 dy\\
    &=\kappa\int_{\calG_z}  q_z(y)\|D(y)(s_{q_z}(y)-s_{\Pi_z}(y))\|^2 dy\\
    &=\kappa \E_{q_z} \|D(s_{q_z}-s_{\Pi_z})\|^2.
\end{align*}
So in combination
\[
(\E_{q_z} [f]-\E_{\Pi_z} [f])^2\leq 8\kappa \E_{q_z} \|D(s_{q_z}-s_{\Pi_z})\|^2. 
\]

Then by Lemma \ref{lem:RN} we have
\begin{align*}
    D\nabla(\log q_z-\log\Pi_z)=D(s_q-s_\pi).
\end{align*}
So we find that 
\[
(\E_{q_z}[f]-\E_{\Pi_z}[f])^2\leq  4\kappa \E_{q_z} \|D(s_{q}-s_{\pi})\|^2. 
\]
Integrating both sides with $q_g(z)$ we find our final claim.

\end{proof}

\begin{proof}[Proof of Proposition \ref{pro:flow}]
Fix a particle $z_t$ in the density of $q_t$, we track its $g$-value trajectory:
\[
\frac{d}{dt}g(z_t)=\nabla g(z_t)\tt v_t(z_t)=-\psi(g(z_t)),
\]
we will show that $g(z_t)\leq M_t$ for all $t$ (The proof for $g(z_t)\geq -M_t$ is identifical is omitted).
Suppose $g(z_t)>M_t+\epsilon$ for some $t$ and $\epsilon>0$. Let $t_0=\inf\{t>0, g(z_t)>M_t+\epsilon\}$. Then $\frac{d}{dt}g(z_{t_0})=-\psi(M_{t_0}+\epsilon)<-\psi(M_{t_0})$, so for a sufficiently small $\delta>0$, $g(z_{t_0-\delta})>M_{t_0-\delta}$, this contradicts the definition of $t_0$.

For the second claim, note that 
\begin{align*}
&\frac{d}{dt} \KL(q_t\|\pi)=-\int q_t(x) v_t(x)\tt(s_\pi(x)-s_{q_t}(x))dx\\
&=-\int q_t(x) (v\perpg(x)+v\parag(x))\tt(s_\pi(x)-s_{q_t}(x))dx\\
&=-F_\bot(q_t, \pi)+\int q_t(x)\frac{\psi(g(x))\nabla g(x)\tt(s_\pi(x)-s_{q_t}(x))}{\|\nabla g(x)\|^2}dx\\
&=-F_\bot(\pi,q_t)+ \int \frac{\psi(g(x))\nabla g(x)\tt s_\pi(x)}{\|\nabla g(x)\|^2}dx\\
&\quad+\int \frac{\|\nabla g(x)\|^2\dot{\psi}(g(x))+\psi(g(x))\Delta g(x)}{\|\nabla g(x)\|^2}q_t(x)dx\\
&\quad-\int \frac{2\psi(g(x))\nabla g(x)\tt\nabla^2 g(x)\nabla g(x)  }{\|\nabla g(x)\|^4} q_t(x)dx\\
&\leq -F_\bot(\pi,q_t)+\E_{q_t} [|\dot{\psi}(g)|]+C_0\E_{q_t} [|\psi(g)|\|\nabla g\|^2].
\end{align*}
\end{proof}

\begin{proof}[Proof of Theorem \ref{thm:grad}]
We use the notation $M_t$ from Proposition \ref{pro:flow}.
When we take $\psi(z)=-\alpha\text{sgn}(z)z^{1+\beta}$, we find that for some $c_0$ that depends on $M_0$
\[
\frac{d}{dt}M_t=-\alpha|M_t|^{1+\beta}\Rightarrow M_t=(\alpha\beta)^{-\frac{1}{\beta}} (t+c_0)^{-\frac{1}{\beta}}.
\]
Moreover, we have that for some constant $C_1$
\begin{align*}
&\frac{d}{dt} \KL(q_t\|\pi)\\
&\leq -F_\bot(q_t,\pi)+\alpha\E_{q_t} (1+\beta)[|g|^\beta]+\alpha C_0\E_{q_t}[|g|^{1+\beta}\|\nabla g\|^2]\\
&\leq -F_\bot(q_t,\pi)+\frac{C_1}{t+c_0}
+\frac{C_1 }{(t+c_0)^{1+1/\beta}}.
\end{align*}
Integrating both sides yields the following for some constant $M_\alpha$
\begin{align*}
\int^T_{T/2} &F_\bot(q_t,\pi)dt\leq
\int^T_{0} F_\bot(q_t,\pi)dt\\
&\leq
\KL(q_0,\pi)+C_1\log\frac{T+c_0}{c_0}+
C_1 M_\alpha.
\end{align*}
So 
\[
\min_{T/2\leq t_0\leq T} F_\bot(q_t,\pi)
\leq \frac2{T}\KL(q_0,\pi)+\frac{2C_1}{T}\log\frac{T+c_0}{c_0}+
\frac{2}{T}C_1 M_\alpha.
\]
Finally, we note that if
$r^D_{q,\pi}$ is the Riez representation of $D(s_\pi-s_{q})^T$, 
\begin{align*}
    \E_q \stein_\pi \phi&=
    \E_q (D(s_\pi-s_q))\tt \phi=\langle r^D_{q,\pi}, \phi\rangle_{\calH}
    \leq \frac12 \|r^D_{q,\pi}\|_{\calH}=\frac12 \sqrt{F_\bot(q,\pi)}. 
\end{align*}

\end{proof}

\section{Additional Experiments Results and Setting Details}
We use NVIDIA GeForce RTX 2080 Ti for neural network experiments.

\subsection{Synthetic Distribution}
For both O-Langevin and O-SVGD, we use $\alpha=100$ and $\beta=0$. We set $\eta=0.01$ and $0.5$ for O-Langevin and O-SVGD respectively. For CLangevin and CHMC, we use a python implementation~\footnote{https://matt-graham.github.io/mici/} and tune the step size and the number of leapfrog steps. We report the best results which are achieved at step size $=0.3$ for CLangevin, and step size $=1$ and number of leapfrog steps $=2$ for CHMC. 

We use energy distance to measure the difference between the approximated distributions by sampling methods and the target distribution. Energy distance is a statistical distance between probability distributions and has been used in the literature, e.g.~\citep{sejdinovic2012hypothesis,sejdinovic2013equivalence,gritsenko2020spectral}. Formally speaking, the energy distance between probability distributions $P$ and $Q$ is defined by
\[
D(P,Q) = 2E_{Z,W}\norm{Z-W}_2 - E_{Z,Z’}\norm{Z-Z’}_2 - E_{W,W’}\norm{W-W’}_2
\]
where $Z, Z’\sim P$ and $W,W’\sim Q$.

\paragraph{Runtime Comparison} 
We report runtime comparison in Figure~\ref{fig:runtime}. We did not include O-SVGD since one iteration of it already $>$ 5s (SVGD is known to take more time per iteration than MCMC due to computing particle interaction). When starting on the manifold, we observe that O-Langevin converges much faster than previous manifold sampling methods. It takes about 4s to fully converge whereas previous methods have not fully converged after 5s. Previous methods cannot work with initializations outside the manifold, thus we only report the runtime of O-Langevin in Figure~\ref{fig:runtime}b. 

\paragraph{Effect of Hyperparameters}
Besides the hyperparameter step size $\eta$ as in standard Langevin and SVGD, our methods have hyperparameters $\alpha>0$ and $\beta\in(0,1]$ in $\psi(x)=\alpha \text{sign}(x)|x|^{1+\beta}$ to control the speed of the sampler to approach the manifold and the closeness of the sampler to the manifold after converging. In theory, as $\alpha$ increases, the sampler approaches the manifold faster and stays closer. As $\beta$ increases, the sampler first converges faster to the manifold but stays relatively far away after converging. We report the results of O-Langevin with varying $\alpha$ (with fixed $\beta=0$) and $\beta$ (with fixed $\alpha=1$) when starting outside the manifold in Figure~\ref{fig:hyperparameters}. We can see that the results of MAE, which measures the closeness to the manifold, align with the theoretical analysis. The energy distance with varying $\alpha$ and $\beta$ are similar while $\alpha=10$ and $\beta=0$ perform slightly better. The theoretical analysis could not tell which values of hyperparameters give the fastest convergence to the target distribution thus we still need to tune $\alpha$ and $\beta$ to achieve the optimal performance. In practice, we recommend to set $\beta=0$ (though our theoretical results apply to $\beta\in(0,1]$, we find $\beta=0$ generally works well in practice.) and tune $\alpha$ to achieve a desirable MAE and energy distance.

\paragraph{Density Estimation}

To compare the estimated density with the ground truth, we plot the collected samples after 5000 epochs when starting on the manifold and 8000 epochs when starting outside the manifold in Figure~\ref{fig:path}. The density estimation from our methods is closer to the ground truth than previous methods, aligning with the results of the energy distance in Figure~\ref{fig:toy}.

\begin{figure}[H]
    \centering
    \begin{tabular}{cccc}		
    	\includegraphics[width=4.2cm]{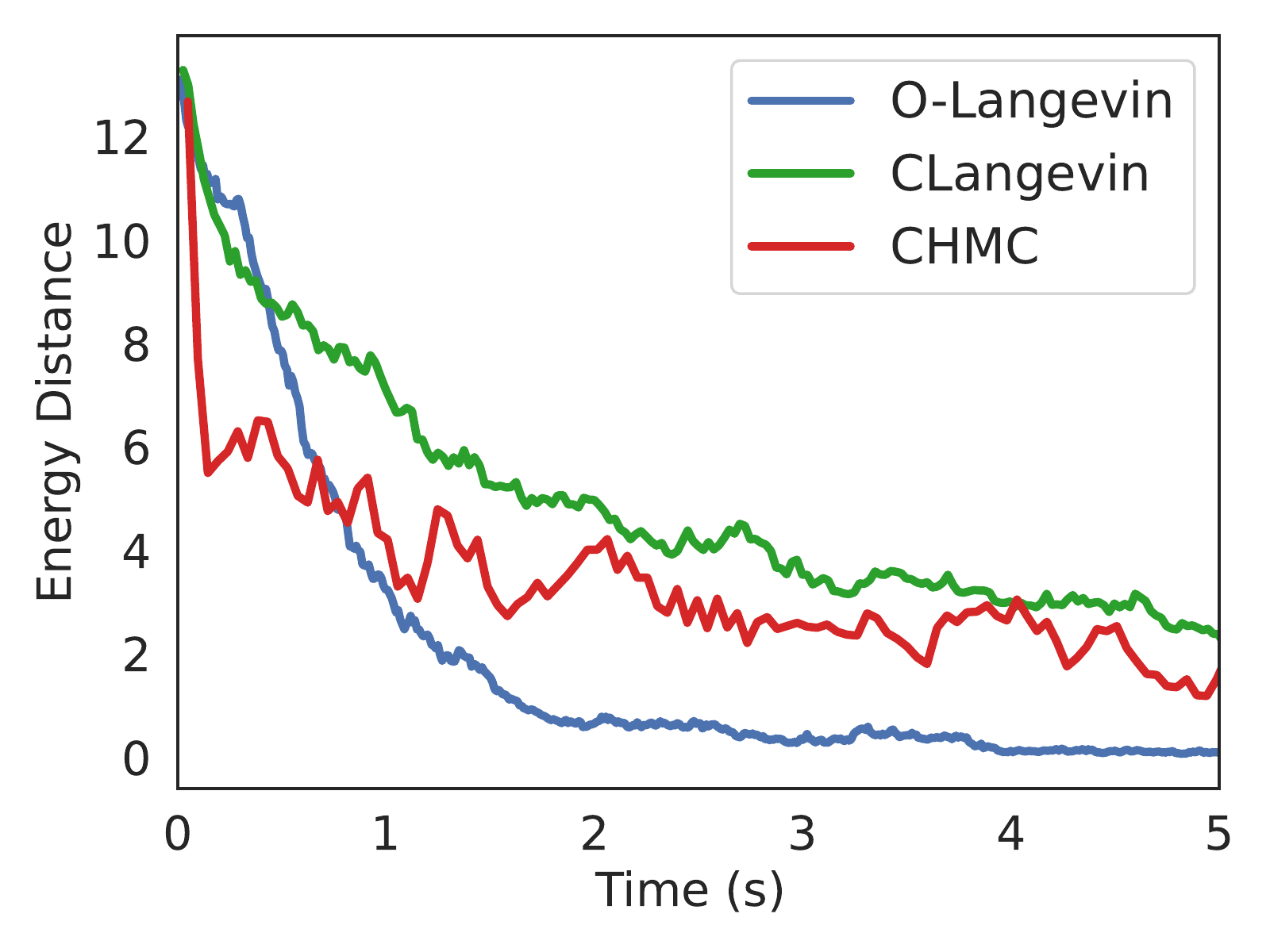} &
    	\includegraphics[width=4.2cm]{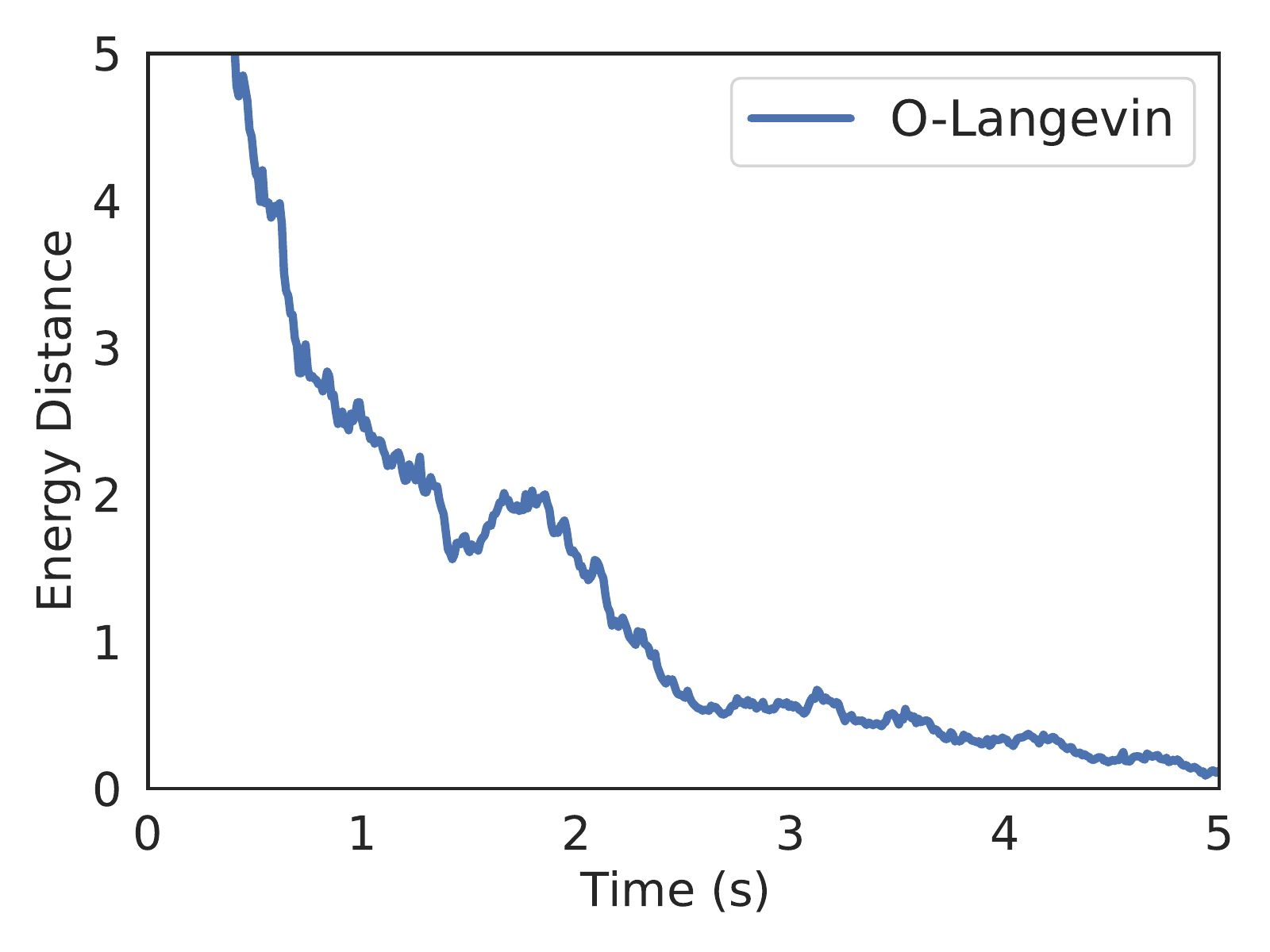} 
    	\\	
    	(a)&
    	(b)\\
    \end{tabular}
    \caption{Runtime comparison when (a) starting on the manifold and (b) starting outside the manifold.}
    \label{fig:runtime}
\end{figure}

\begin{figure}[H]
    \centering
    \begin{tabular}{cccc}		
    	\includegraphics[width=4.2cm]{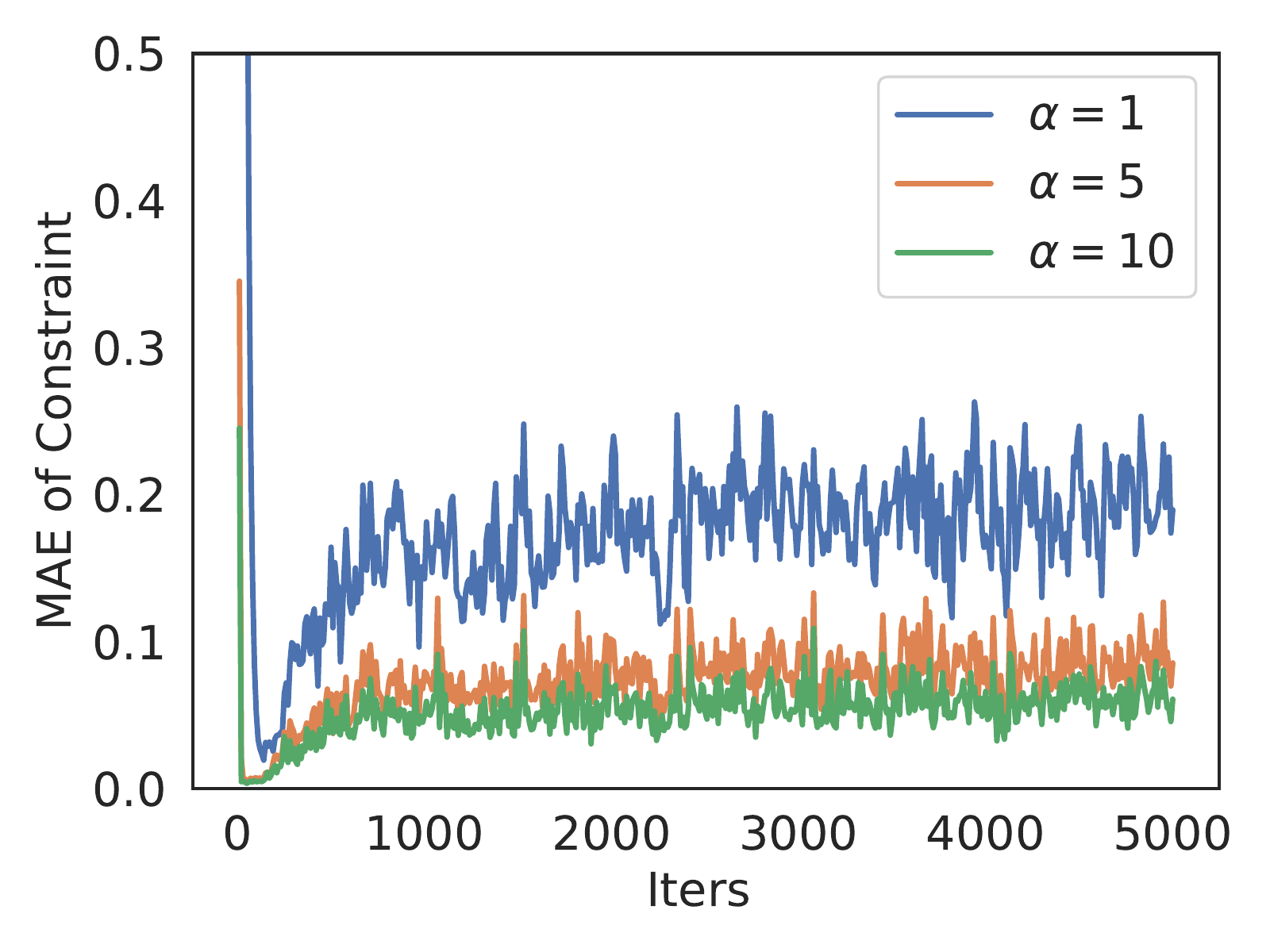} &
    	\includegraphics[width=4.2cm]{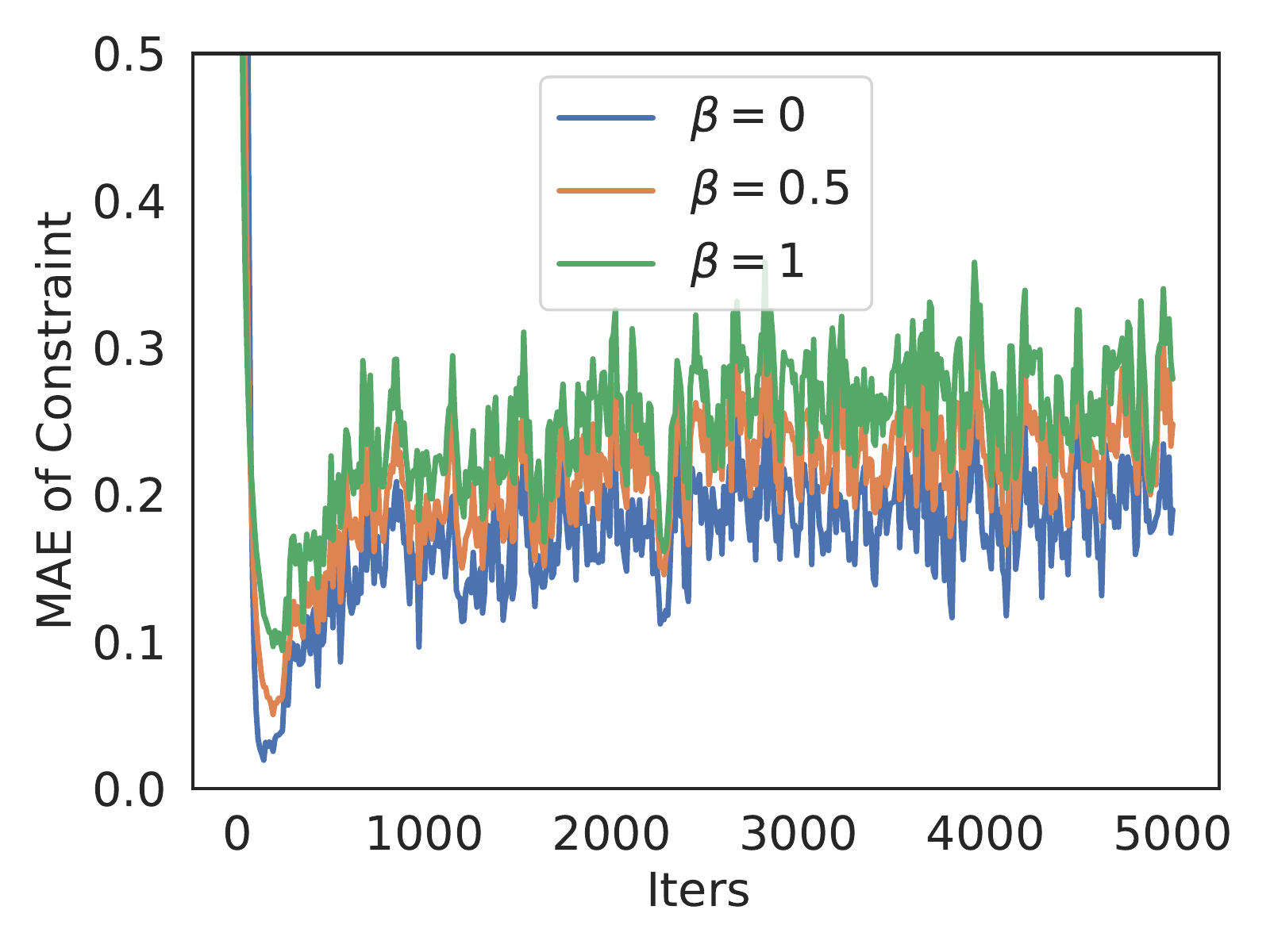}  &\\
    	\includegraphics[width=4.2cm]{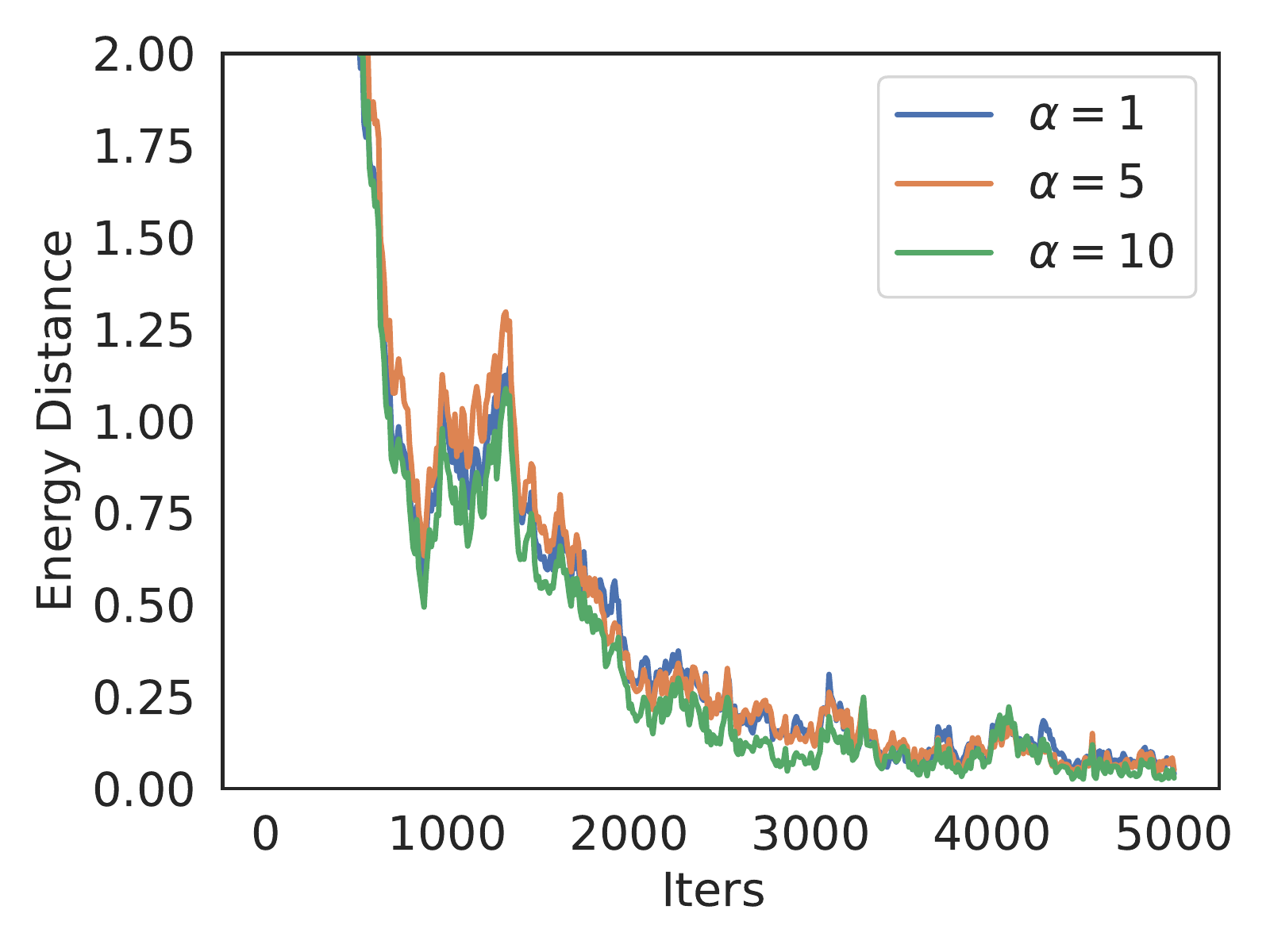} &
    	\includegraphics[width=4.2cm]{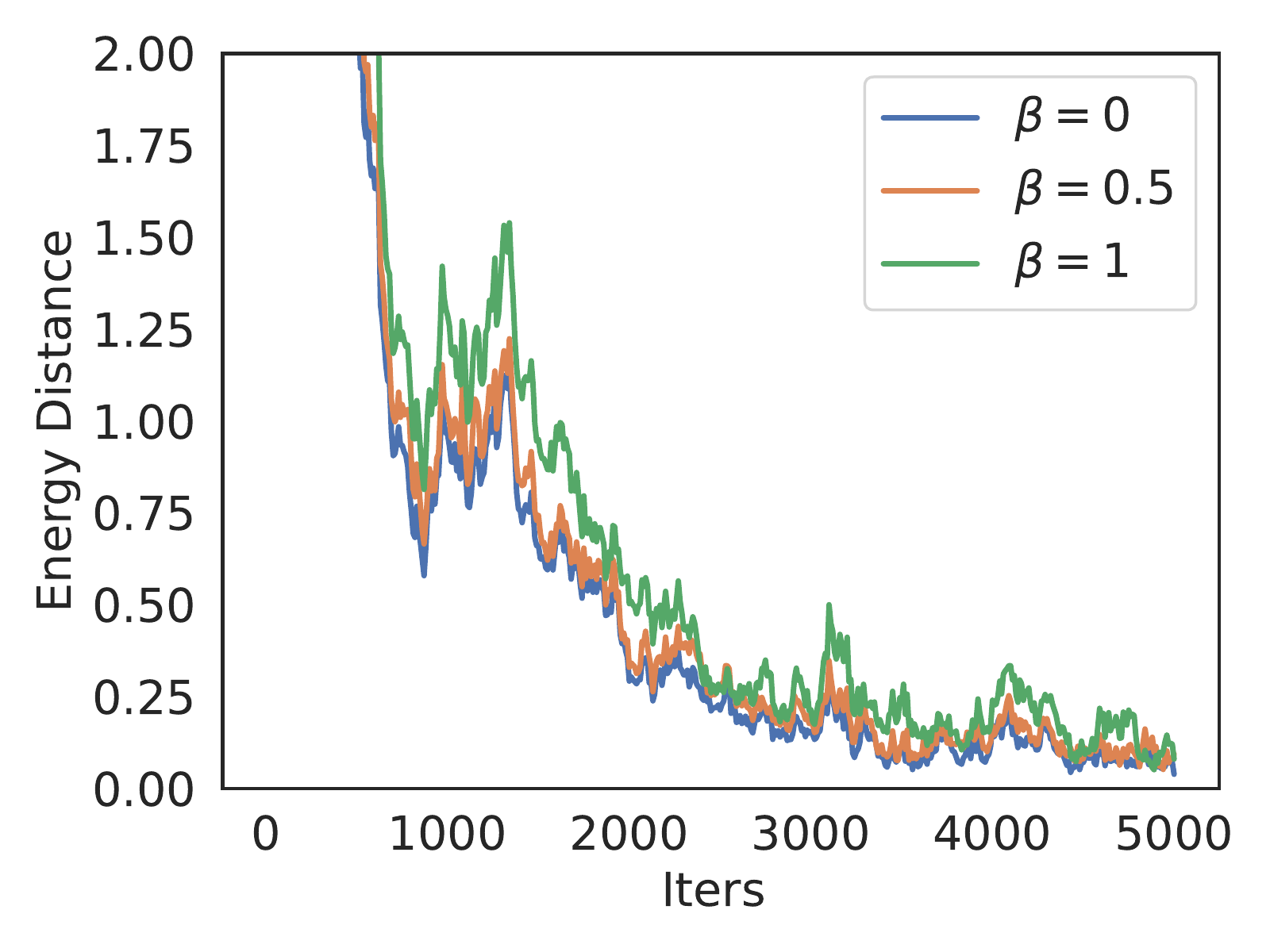}  
    	\\	
    	\hspace{-0mm}\\		
    \end{tabular}
    \caption{Effect of hyperparameters $\alpha$ and $\beta$}
    \label{fig:hyperparameters}
\end{figure}

\begin{figure}[H]
    \centering
    \setlength{\tabcolsep}{1pt}
    \begin{tabular}{ccc}		
    	&\includegraphics[width=4.2cm]{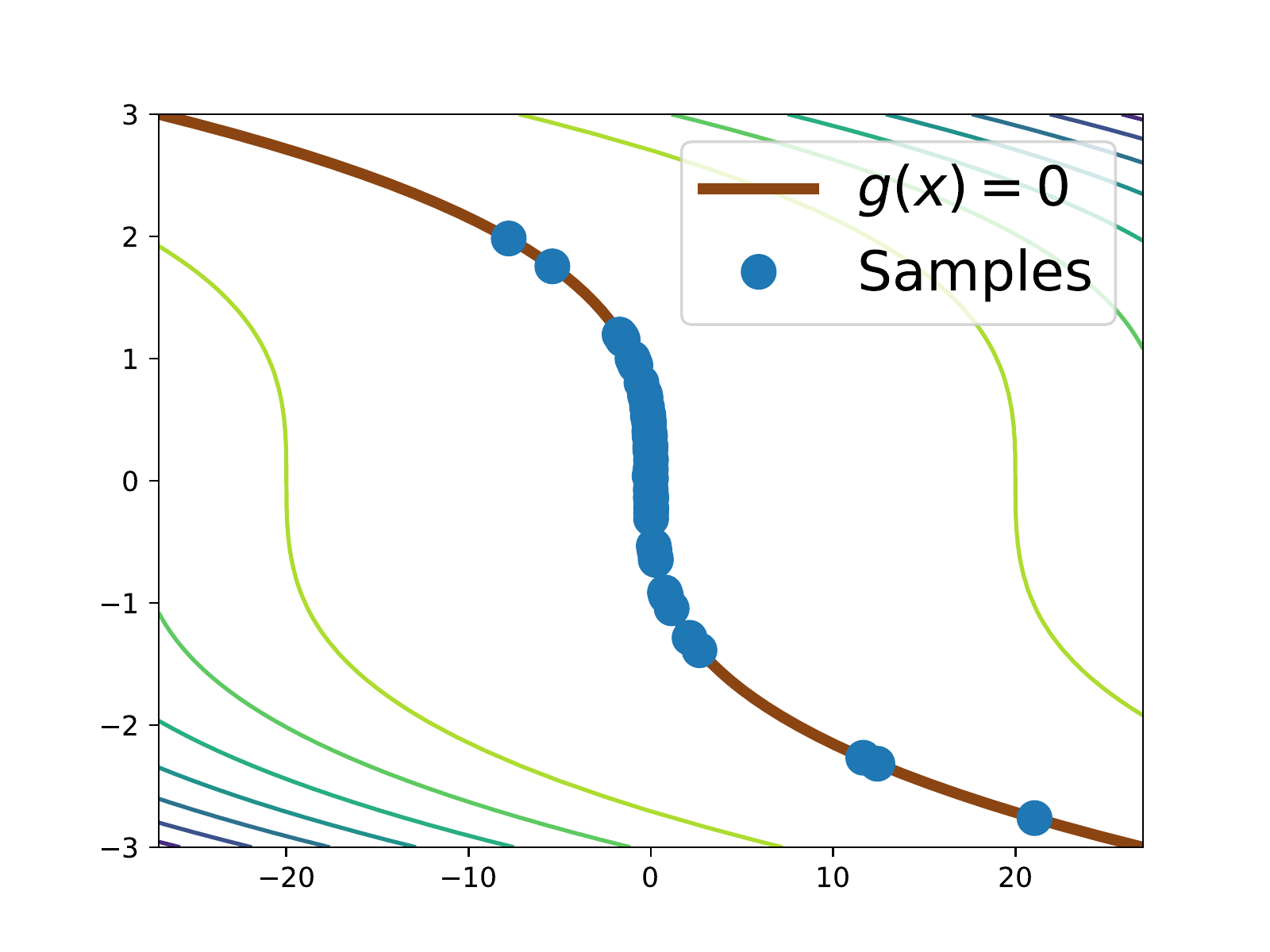} &
    	\includegraphics[width=4.2cm]{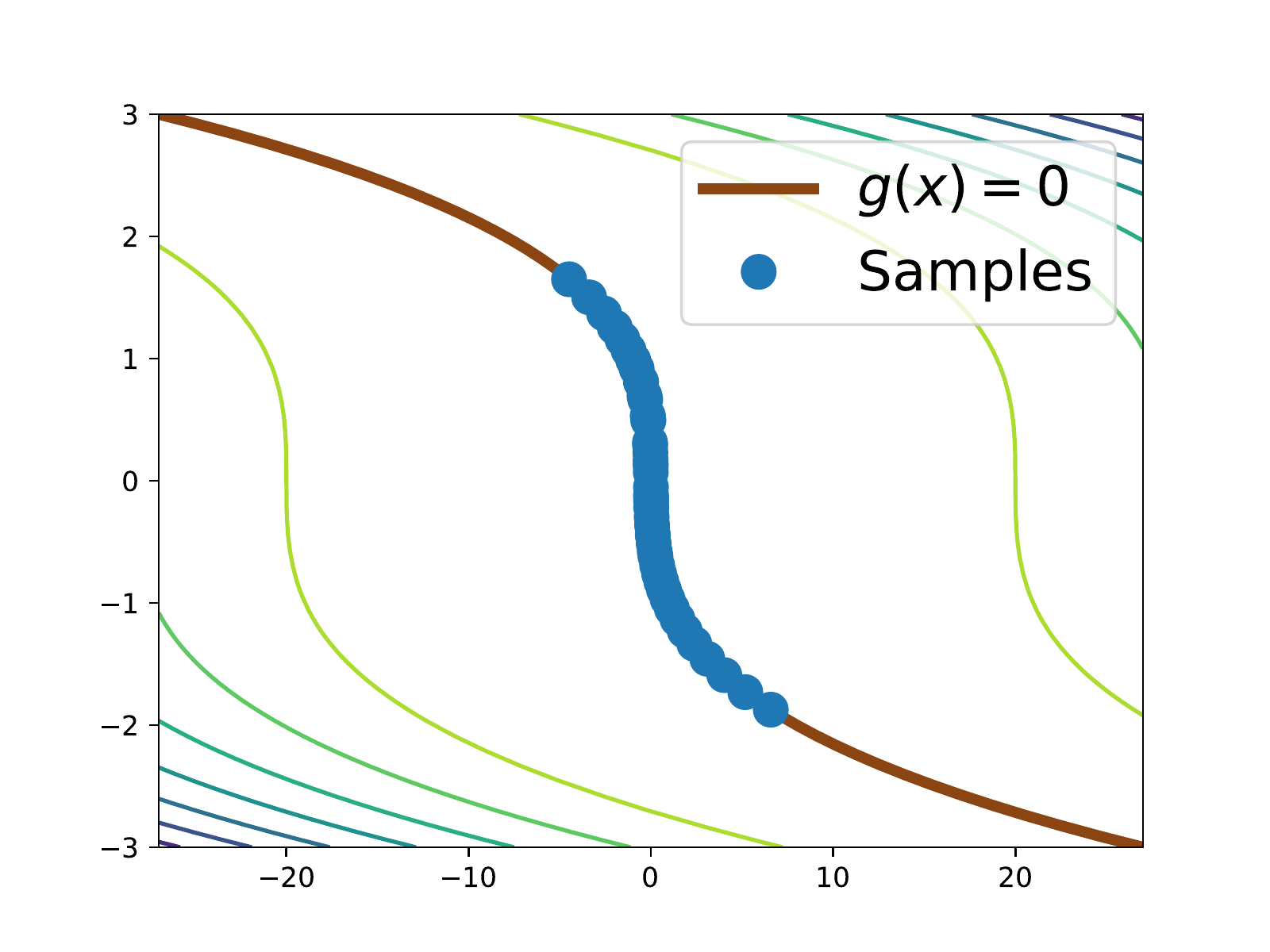}
    	\\	
    	&(a) O-Langevin 
    	&(b) O-SVGD \\
    	\includegraphics[width=4.2cm]{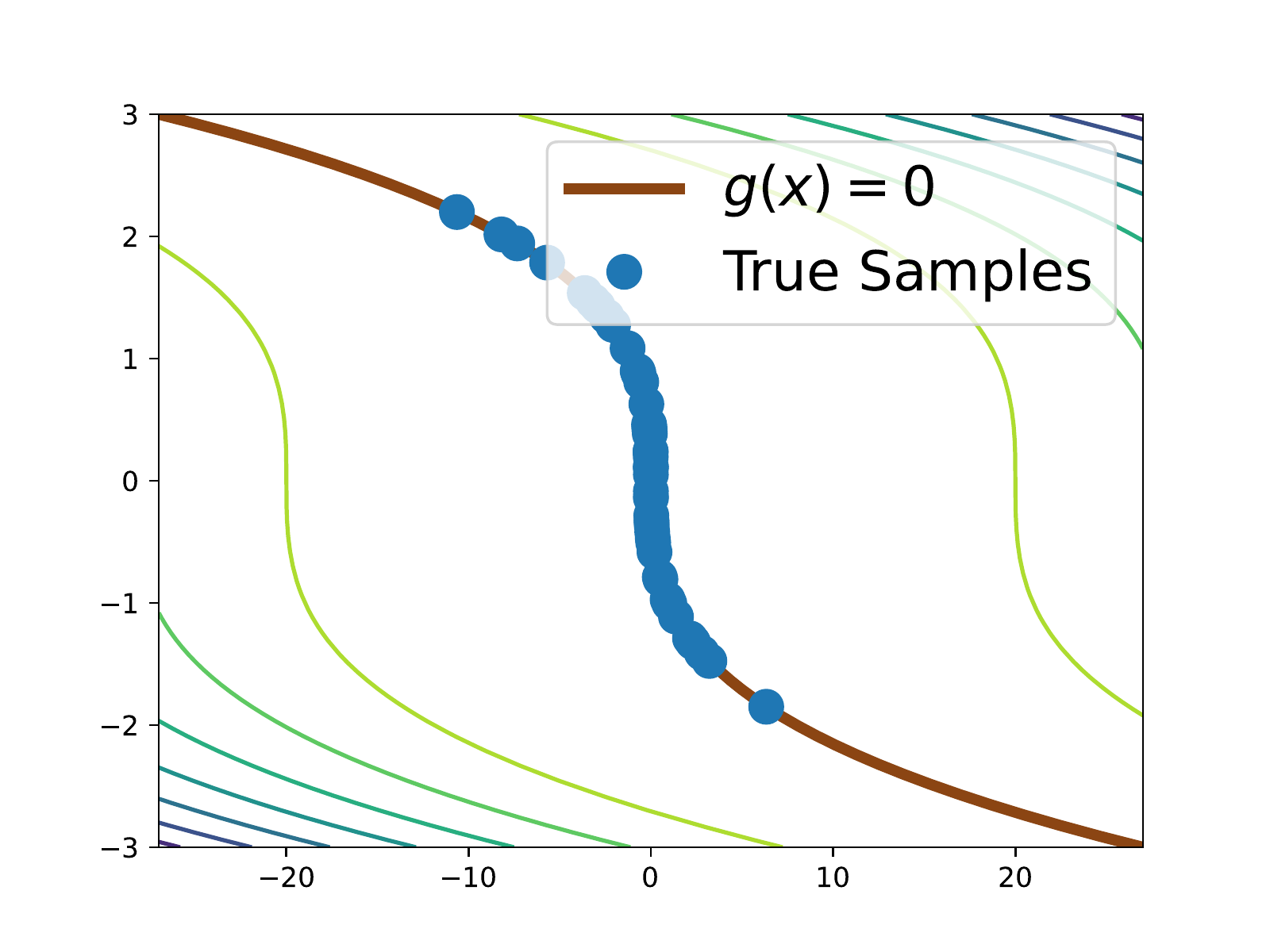}  &
    	\includegraphics[width=4.2cm]{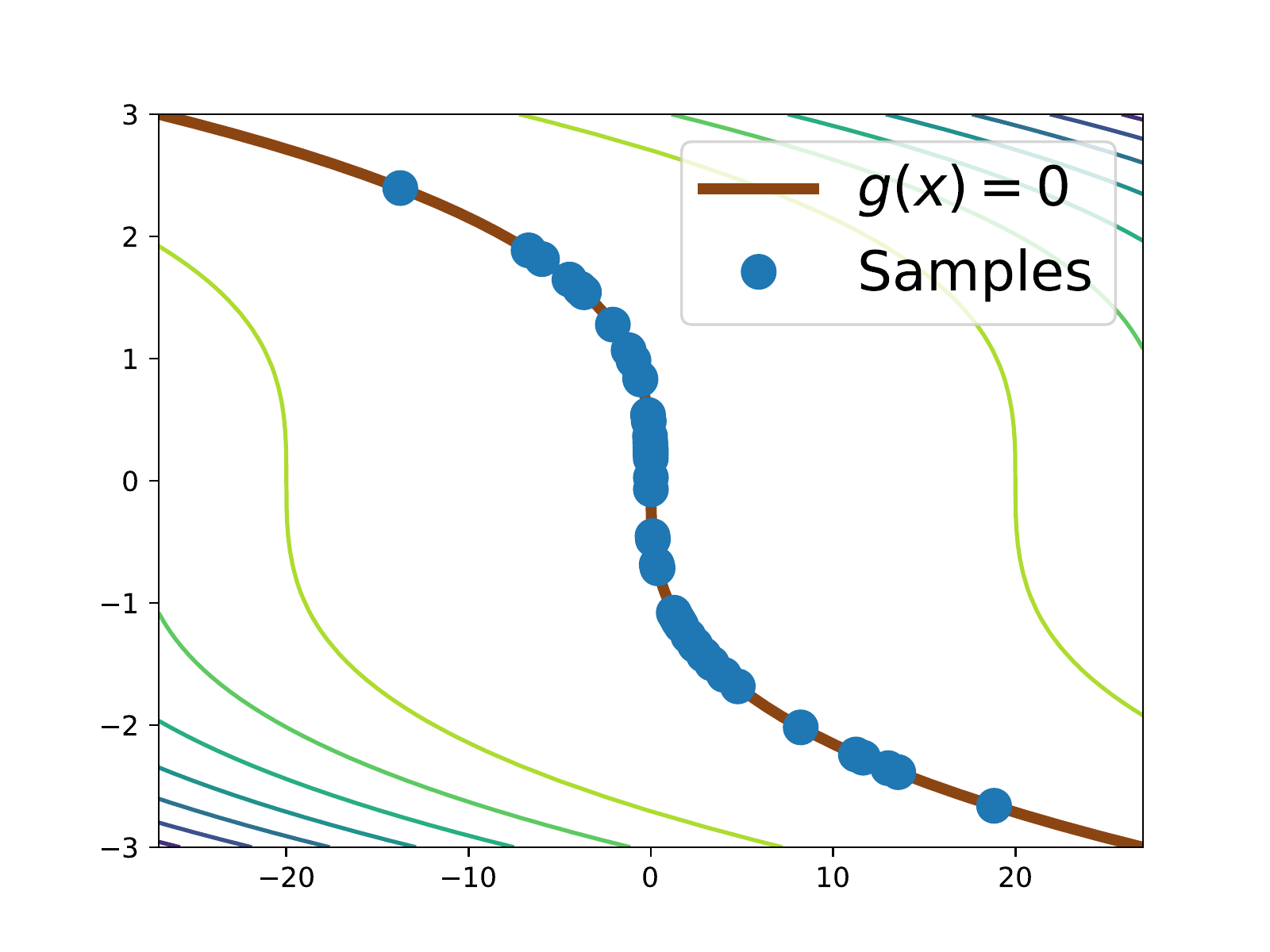}  &
    	\includegraphics[width=4.2cm]{figs/toy/path_clangevin_True_on.pdf} \\
    	Ground Truth
    	&(c) CLangevin 
    	&(d) CHMC \\
    	&
    	\includegraphics[width=4.2cm]{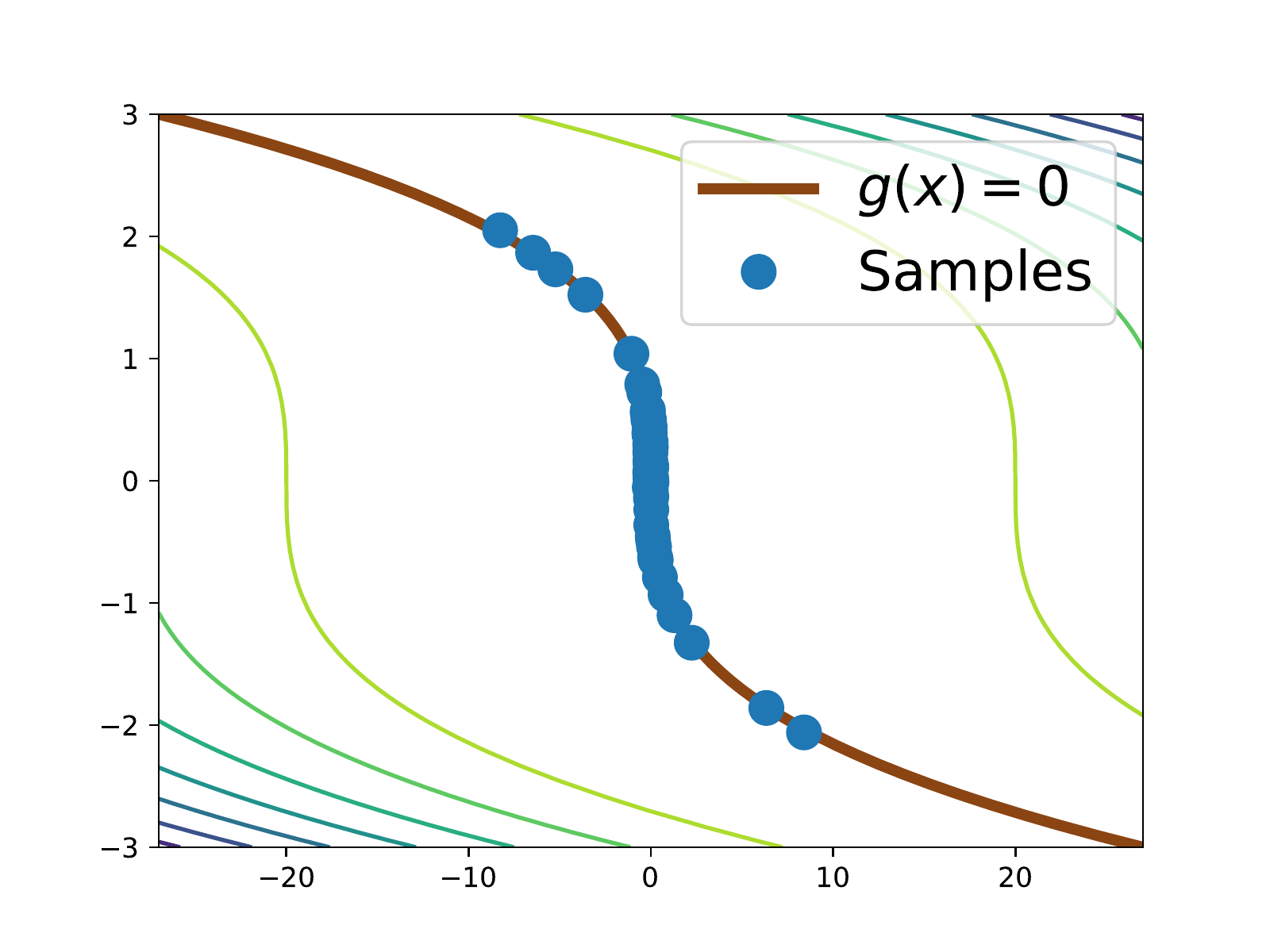} &
    	\includegraphics[width=4.2cm]{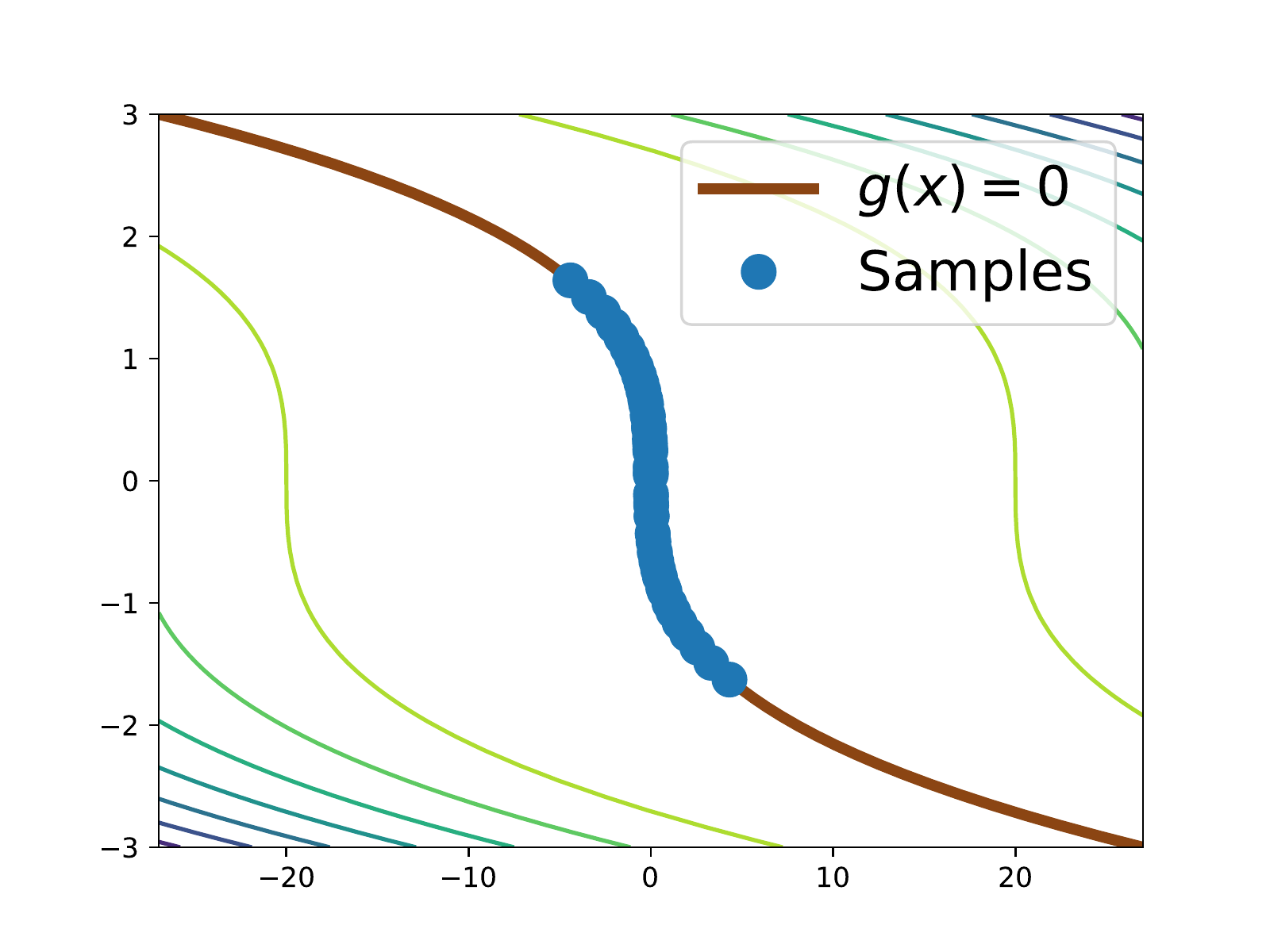}
    	\\	
    	&(e) O-Langevin
    	&(f) O-SVGD \\
    \end{tabular}
    \caption{Density estimation when (a)-(d) starting on the manifold and (e)-(f) starting outside the manifold. 
    }
    \label{fig:path}
    \end{figure}

\subsection{Income Classification with Fairness Constraint}

The Adult Income dataset contains 30,162 training samples and 15,060 test samples. The feature dimension is 86. Following previous work~\citep{martinez2020minimax,liu2020accuracy}, we obtain the training set by randomly subsampling 20,000 data points from the training samples. The model is a two-layer multilayer perceptron (MLP), which has 50 hidden units and RELU nonlinearities. The metric values are the mean over all particles. For both methods, we use $n=10$ particles and $\beta=0$ . For O-Langevin, $\alpha=100$ and $\eta=10^{-5}$. For O-SVGD, $\alpha=130$ and $\eta=10^{-4}$. The results are averaged over 3 runs with the standard error as the error bar.

\subsection{Loan Classification with Logic Rules}
The dataset\footnote{https://www.kaggle.com/wendykan/lending-club-loan-data} contains loans issued through 2007-2015 of several banks. Each data point contains 28 features such as the current loan status and latest payment information. We define the logic loss to be the binary cross-entropy loss. The metric values are the mean over all particles. For both methods, we use $n=10$ and $\beta=0$ . For O-Langevin, $\alpha=80$ and $\eta=10^{-4}$. For O-SVGD, $\alpha=100$ and $\eta=10^{-3}$. The results are averaged over 3 runs with the standard error as the error bar.

\subsection{Prior-Agnostic Bayesian Neural Networks}
For large models, such as ResNet-18 on this task, computing second-order derivatives is slow. To speed up our methods on large models, we ignore the second-order terms in O-Langevin and O-SVGD and empirically find that they still perform well. We leave the theoretical analysis for future work.

Specifically, we ignore the $r$ term in the update of O-Langevin and obtain
\begin{align*}
    x_{t+1} &= x_t + \eta\cdot v_\sharp(x_t) + \Langevin\perpg(x_t),\nonumber\\
    \text{where}~~\Langevin\perpg(x_t) &= \eta D(x_t)\nabla\log\pi(x_t) + \sqrt{2\eta}D(x_t)\xi_t,~~\xi_t\sim\mathcal{N}(0,I).
\end{align*}
For O-SVGD, the update becomes
\begin{align*}
    x_{i,t+1} &= x_{i,t} + \eta\cdot \left (v_\sharp(x_{i,t}) +   \SVGD_{K\perpg}(x_{i,t})\right),\nonumber\\
    \text{where}~~ \SVGD_{K\perpg} (x_{i,t}) &=\frac{1}{n}\sum_{j=1}^n k\perpg(x_{i,t},x_{j,t})\nabla_{x_{j,t}}\log\pi(x_{j,t})+\tilde{\nabla}_{x_{j,t}}k\perpg(x_{i,t},x_{j,t})\\
    \text{and}~~\tilde{\nabla}_{x_{j,t}}k\perpg(x_{i,t},x_{j,t})&=D(x_{i,t})(D(x_{j,t})\nabla_{x_{j,t}}k(x_{i,t},x_{j,t}))
\end{align*}

For all results, we use $200$ epochs, $64$ batchsize, $n=4$, $\beta=0$, $\alpha=1000$ and $\eta=10^{-4}$. During testing, we do Bayesian model averaging to obtain test error, ECE and AUROC.

\subsection{Computational Cost Comparison}
Our method is the first constraint sampling without the requirement of initialization on the manifold, so there is essentially no baseline that can achieve the same effect. Compared to the unconstrained Langevin and SVGD, our method additionally computes the gradient and the Hessian of the constraint function. Compared to previous manifold sampling methods which require expensive projection subroutines, our method has a much cheaper and faster update. For example, in the synthetic distribution experiment, one update of O-Langevin (ours) takes 0.023s whereas the previous method CLangevin takes 0.08s. From Figure~\ref{fig:toy}a, we can see that O-Langevin also converges faster than CLangevin in terms of the number of iterations.

\subsection{Further Comparison to Previous Methods}
\paragraph{Manifold Sampling Methods}
Previous manifold sampling methods assume that the initialization is on the manifold. One may wonder if we can obtain such an initialization by optimization algorithms so that we can still use previous methods when the sampler starts outside the manifold. This will not work because the initialization must be exactly on the manifold whereas the solutions found by optimization always have some intolerable error. Finding a point that is exactly on the manifold without any prior knowledge is by itself a hard problem. Therefore, we are not able to compare our methods with previous methods when there are no known in-domain points, such as the income, loan and image classification tasks in Section~\ref{sec:exp}. 

\paragraph{Moment Constraints}
As mentioned in the related work, sampling with moment constraints $\E_q[g]$ cannot guarantee every sample to satisfy the constraint. To empirically show the difference between our methods and this type of methods, we compare O-Langevin to Control+ Langevin, which is a recently proposed moment constraint sampling method~\citep{liu2021sampling}, on the income classification task. We report the mean and the maximum value of the fairness loss in Figure~\ref{fig:moment-constraint}. While both methods have small mean fairness loss, the maximum value of O-Langevin is always much smaller than that of Control+Langevin. This suggests that every sample of our method satisfies the constraint well whereas some samples of Control+Langevin violate the constraint significantly, since the moment constraint can only guarantee the mean value instead of the value of each sample.

\begin{figure}[H]
    \centering
    \begin{tabular}{cccc}		
    \includegraphics[width=5cm]{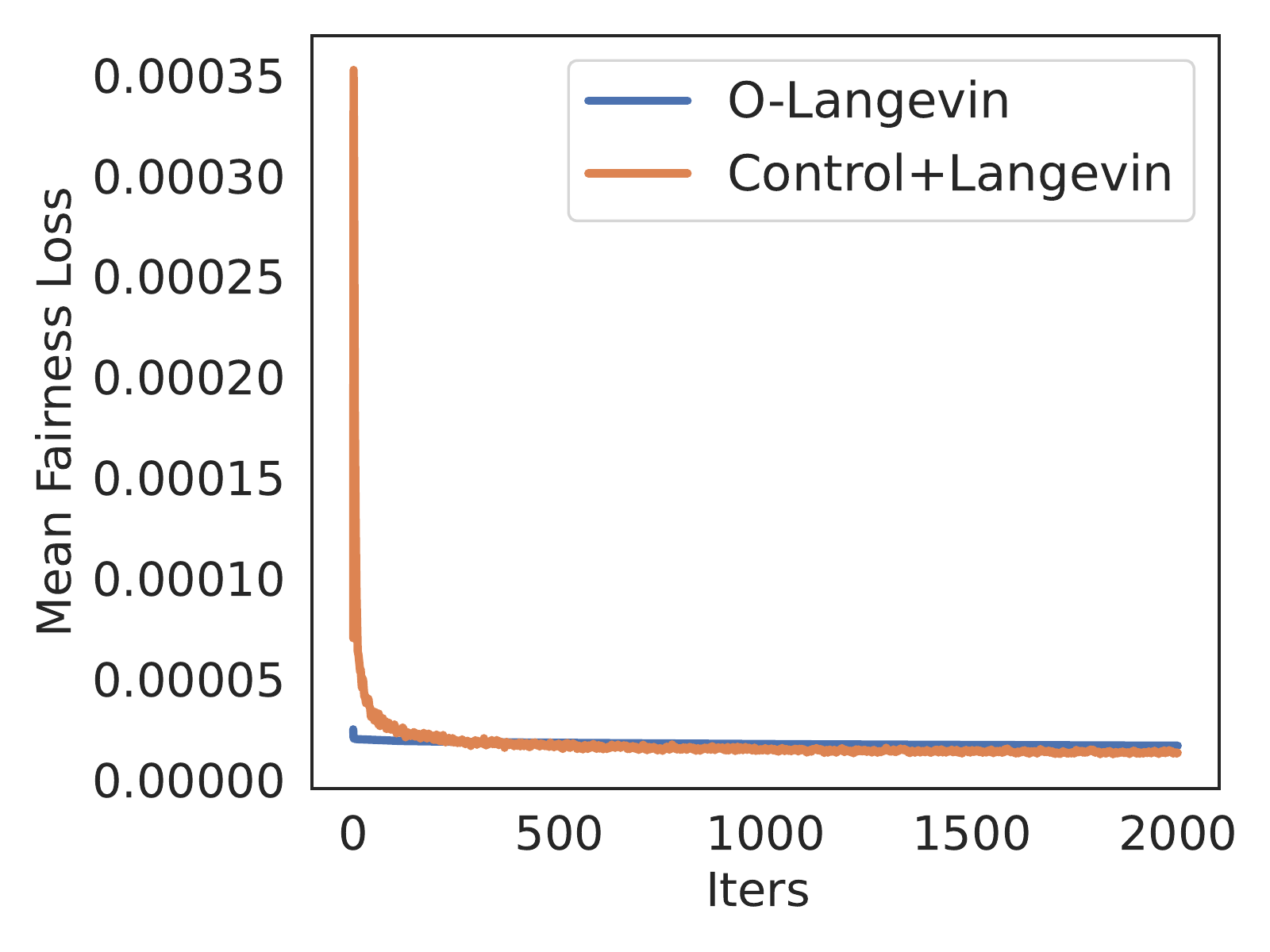} 
    	\includegraphics[width=5cm]{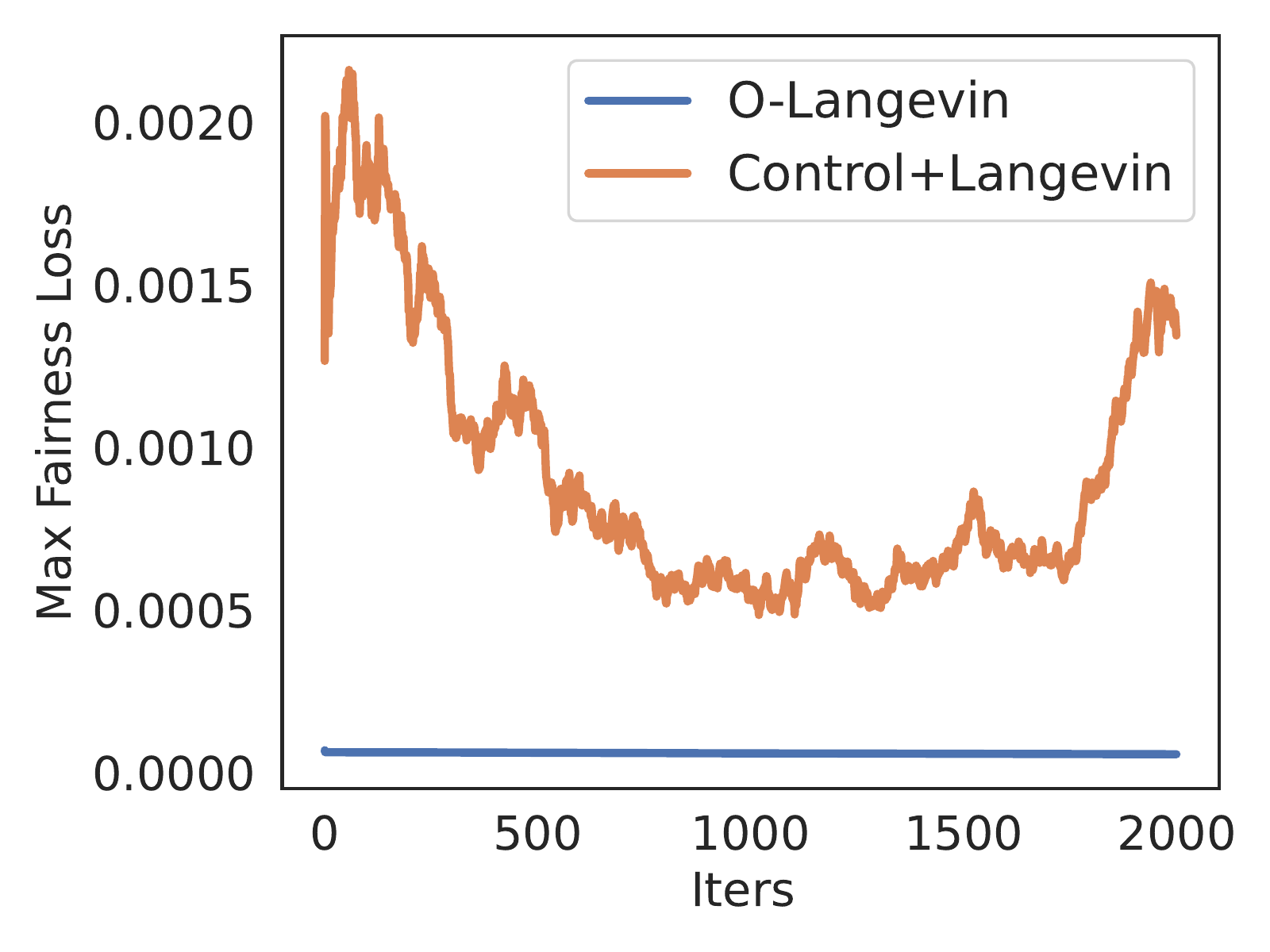} 
    \end{tabular}
    \caption{Every sample of O-Langevin satisfies the constraint well whereas some samples of Control+ Langevin violate the constraint significantly.}
    \label{fig:moment-constraint}
\end{figure}

\end{document}